\documentclass{article}

\usepackage{arxiv}

\usepackage{microtype}
\usepackage{graphicx}
\usepackage{subcaption}
\usepackage{booktabs}
\usepackage{array}
\usepackage{hyperref}
\usepackage{amsmath}
\usepackage{amssymb}
\usepackage{mathtools}
\usepackage{amsthm}
\usepackage[capitalize,noabbrev]{cleveref}
\usepackage{natbib}
\usepackage{tikz}
\usepackage{tikz-cd}
\usepackage{pgfplots}

\usepackage{algorithm}
\usepackage{algorithmic}

\usepackage{xcolor}
\usepackage{enumitem}
\RequirePackage{doi}
\usetikzlibrary{arrows.meta,calc,positioning,decorations.pathreplacing,fit}
\pgfplotsset{compat=1.18}

\hypersetup{
  pdftitle={Latent-Confounded Causal Discovery via Lie-Bracket Geometry},
  pdfsubject={Causal discovery and information geometry},
  pdfauthor={Anonymous Authors},
  pdfkeywords={causal discovery, information geometry, intervention response fields, Lie brackets, geometric screening},
  colorlinks=true,
  linkcolor=blue!55!black,
  citecolor=blue!55!black,
  urlcolor=blue!55!black
}

\theoremstyle{plain}
\newtheorem{theorem}{Theorem}[section]
\newtheorem{proposition}[theorem]{Proposition}

\newtheorem{corollary}[theorem]{Corollary}
\theoremstyle{definition}
\newtheorem{definition}[theorem]{Definition}

\theoremstyle{remark}

\usepackage[utf8]{inputenc}
\usepackage{mathtools}
\usepackage{booktabs}
\usepackage{hyperref}
\usepackage{tikz-cd}
\usepackage{bm}

\newcommand{\Kl}{\mathsf{Kl}} % Kleisli category notation
\newcommand{\Dist}{\mathsf{Dist}}  % Distribution (probability) monad
\newcommand{\R}{\mathbb{R}}
\newcommand{\E}{\mathbb{E}}
\newcommand{\M}{\mathcal{M}}
\newcommand{\C}{\mathcal{C}}
\newcommand{\D}{\mathcal{D}}
\newcommand{\Lan}{\text{Lan}}
\newcommand{\Ran}{\text{Ran}}
\newcommand{\dobs}{P_{\text{obs}}}
\newcommand{\ddo}{P_{\text{do}}}

\title{\textbf{
Latent-Confounded Causal Discovery via Lie-Bracket Geometry}}
\author{Sridhar Mahadevan\\
Adobe Research and University of Massachusetts, Amherst\\
\texttt{smahadev@adobe.com, mahadeva@umass.edu}}

\date{\today}

\begin{document}

\maketitle

\begin{abstract}
We study causal discovery from observational and interventional regimes when latent variables may affect the measured system. Radon--Nikodym ratios between regime distributions, or learned transports realizing the same changes of measure, provide local response fields on a statistical manifold. Lie brackets of the learned fields then quantify whether their visible span is locally closed. Because a non-closing bracket can also arise from target mismatch, domain shift, weak overlap, insufficient adjustment, field-estimation error, or omitted coordinates, we use it as a diagnostic rather than as a certificate of hidden confounding.

Our first algorithm, BRIDGE (Bracket Residuals for Interventional Discovery and Geometric Estimation), combines a density-ratio or transport engine with a high-recall geometric screen and passes the retained arrows to a score-based or differentiable discovery method. The main formulation and experiments use known single-node intervention targets; in that regime the screen is designed to retain candidate directed effects, while a downstream learner determines the final graph or equivalence-class representation. Our second algorithm, Spectral Kernel Flow Matching (\textsc{SKFM}), amortizes the response fields, summarizes residual nonclosure by a spectral visible-footprint subspace, and applies an order-dependent graph extractor. Direct extraction succeeds on calibrated chains and selected motifs, but is unstable on harder random DAGs when the order must be learned. On ten-node nonlinear random DAGs, the more reliable hybrid role of the geometry is as a candidate generator: calibrated \textsc{SKFM}/\textsc{Bridge} fields followed by local BIC scoring achieve mean directed $F_1\simeq0.86$. Sachs protein signaling provides a real-data stress test and supports a diagnostic, not fully identified, interpretation. The contribution is therefore a practical interventional screening pipeline, explicit guarantees for screen retention and residual-footprint rank under stated assumptions, and a falsifiable account of the boundary between geometric diagnostics and causal identification.
\end{abstract}


\section{Introduction}
\label{sec:introduction}

Causal inference \citep{rubin-book,pearl-book} distinguishes passive observations from data collected after a system has been perturbed. This paper asks whether the local geometry of those regime changes can help causal discovery when some causes are unobserved. If an interventional law $P^{(e)}$ is absolutely continuous with respect to an observational reference law $P^{(0)}$, its Radon--Nikodym ratio describes the change of measure. Differentiating a smooth density-ratio model or a learned transport yields a response field. Lie brackets of several such fields test whether the visible distribution they span is involutive and hence locally integrable in the sense of Frobenius.

The interpretation requires care. Fixed hard interventions on distinct structural assignments commute at the level of graph surgery: applying $\mathrm{do}(X_i=x_i)$ and $\mathrm{do}(X_j=x_j)$ gives the same intervened SCM in either order. The fields studied here are instead estimated transports between empirical regimes. Their bracket can be nonzero because the transports depend on state, because their chosen extensions off the observed regimes differ, or because the visible coordinates do not form a closed description. Latent common causes are one possible source, but so are unknown targets, general domain shifts, support mismatch, insufficient adjustment, model error, and missing observed coordinates. Our central claim is therefore a screening claim: bracket geometry can expose where a visible response-field model fails to close and can help reduce a downstream discovery search. It does not turn noncommutativity into a latent-confounding oracle.

\subsection{Data Assumptions and Output}
\label{sec:data-assumptions-output}

Let $V=\{X_1,\ldots,X_d\}$ be the observed variables and let $L$ denote arbitrary unobserved variables. The latent projection of an underlying acyclic SCM onto $V$ is represented by a maximal ancestral graph (MAG); the invariant marks shared by observationally equivalent MAGs are represented by a PAG. We distinguish three inputs because they imply different estimands.

\subsubsection{\textsc{Bridge}-KT (Known Targets)}
We observe a reference sample $\mathcal D^{(0)}\sim P^{(0)}$ and regime samples $\mathcal D^{(e)}\sim P^{(e)}$, with target set $I_e\subseteq V$ supplied for each regime. The experiments primarily use one soft or hard single-node target per regime. The geometric layer returns a directed candidate mask, field-calibration diagnostics, and bracket-residual scores. A downstream latent-aware method should report PAG/MAG features; a downstream DAG scorer instead returns a selected visible approximation under its modeling restrictions. A unique underlying DAG is not claimed without additional identifying assumptions.

\subsubsection{\textsc{Bridge}-UT (Unknown Targets)}
Here the distributions are labeled by regime but $I_e$ is unknown. The pair consisting of the causal graph and intervention targets is generally identifiable only up to $\Psi$-Markov equivalence. Accordingly, the natural population output is a $\Psi$-PAG, as in $\Psi$-FCI \citep{jaber2020soft}, together with target information invariant across that class. A bracket or density-shift score may propose candidate targets, but it does not upgrade this output to a unique DAG or prove target recovery. The present theorems and controlled experiments do not establish \textsc{Bridge}-UT consistency.

\subsubsection{\textsc{Bridge}-Domain (General Domains)}
If regime labels index environments whose mechanism changes are not known to be interventions on a shared SCM, the inputs are domain shifts rather than interventional samples. The output is then a stability or transport diagnostic unless assumptions such as those of Joint Causal Inference connect context variables to mechanisms. We do not interpret arbitrary domain differences as do-interventions.

\subsubsection{Visible Graph Terminology}
Throughout, ``latent projection'' means the MAG induced on $V$. ``Visible DAG'' refers only to the DAG selected by a restricted score-based approximation over observed nodes. These objects are not interchangeable: a selected DAG can be useful for prediction or model comparison while failing to encode the equivalence-class uncertainty and bidirected structure of a PAG/MAG. Likewise, a large projected bracket residual is not itself a bidirected edge.

\subsubsection{Measure-Change Interface}
For regimes satisfying $P^{(e)}\ll P^{(0)}$, define
\[
  \rho_e(z)=\frac{dP^{(e)}}{dP^{(0)}}(z),
  \qquad
  \mathbb E_{P^{(e)}}[f(Z)]
  =
  \mathbb E_{P^{(0)}}[\rho_e(Z)f(Z)].
\]
This is a change-of-measure identity, not an identification theorem. Absolute continuity and adequate overlap are substantive requirements; when they fail, ratios can be undefined or unstable.

\subsection{Related Work} 

Existing causal discovery algorithms are commonly grouped into constraint-based, score-based, hybrid, functional-model, and continuous optimization families. PC/FCI-style methods use conditional-independence and separation tests, often representing latents through partial ancestral structure \citep{spirtes2000causation,cd-survey}; RFCI and GFCI variants trade off additional testing and score guidance for scalable latent-aware PAG discovery \citep{colombo2012learning,ogarrio2016hybrid}; GES-style methods search DAGs or equivalence classes under decomposable scores such as BIC \citep{chickering:jmlr}; and NOTEARS/DCDI-style methods optimize smooth edge weights with acyclicity penalties, gradients, or normalizing-flow likelihoods \citep{notears,lachapelle2020gradient,brouillard2020differentiable}.

The output depends on the data regime and assumptions. With latent variables, observational constraint-based discovery generally targets a PAG representing MAG invariants, not a unique visible DAG. With multiple soft-intervention distributions and unknown targets, the appropriate population target is the $\Psi$-Markov equivalence class of graph--target pairs and its $\Psi$-PAG representation; $\Psi$-FCI is sound and complete for those invariant endpoint marks under its assumptions \citep{jaber2020soft}. Joint Causal Inference instead augments the variable set by context variables and makes different assumptions about regime generation. Our known-target experiments are not an unknown-target identification result.

\textit{NOTEARS} is especially close in motivation to the present paper because it also begins from the super-exponential size of the DAG search space \citep{notears}. Its central move is to replace the discrete acyclicity constraint by a smooth and exact trace characterization over real-valued adjacency matrices, so that structure learning can be treated as continuous optimization. \textsc{Bridge} takes a complementary route. Rather than relaxing the graph search into a global matrix program, it estimates local interventional vector fields and uses Lie/Frobenius compatibility to compress the admissible arrow family before a downstream scorer is run. Thus both approaches attack the same combinatorial bottleneck, but at different levels: NOTEARS differentiates the acyclicity constraint itself, whereas \textsc{Bridge} changes the candidate family being scored by inserting an information-geometric screen sensitive to interventions, latent obstruction, and regime mismatch.

A particularly relevant functional-model line is LiNGAM: linear non-Gaussian acyclic modeling uses non-Gaussian independent disturbances to orient causal directions that are not identifiable from covariance structure alone \citep{shimizu2006lingam}. ICA-LiNGAM and DirectLiNGAM estimate causal order and edge weights under causal sufficiency \citep{shimizu2011directlingam,hyvarinen2013pairwise}; the modern Python \texttt{lingam} package collects these methods together with latent-confounder variants \citep{ikeuchi2023lingam}. For hidden variables, ParceLiNGAM seeks causal order robust to latent confounders \citep{tashiro2014parcelingam}, RCD repetitively extracts ancestors and parent relations in linear non-Gaussian models with latent confounders \citep{maeda2020rcd}, CAM-UV extends additive-noise discovery to unobserved variables \citep{maeda2021camuv}, and ABIC-LiNGAM formulates a differentiable LiNGAM objective with unmeasured confounding \citep{morinishi2025abic}. Closely related cumulant methods exploit higher-order moments in latent-variable LiNGAM models, including directed latent-confounded discovery and causal-effect identification under stronger structural assumptions \citep{cai2023cumulants,tramontano2025ceid}. Our contribution is orthogonal to this family: Lie/Frobenius geometry supplies a structural screen and latent-obstruction diagnostic before a downstream score-based, differentiable, or functional-model method chooses a final graph.

The distinction matters because equivalence-class search does not by itself remove the combinatorial scale of score-based discovery. Covered edge reversals collapse observationally indistinguishable DAGs into Markov equivalence classes, but the quotient is only a constant-factor reduction in the known counting results: Gillispie and Perlman's enumeration suggests a limiting DAG-to-class ratio of about $3.7$, and Schmid and Sly prove that the ratio of Markov equivalence classes to DAGs converges to a positive constant \citep{gillispie2002size,schmid2024number}. Thus a GES-style equivalence-class search still faces the same super-exponential order of growth; the geometric question is whether one can shrink the admissible arrow family before scoring begins.

\subsection{\textsc{Bridge}: Modeling Intervention through Infinitesimal Flows}

\begin{figure}[t]
\centering
\resizebox{\textwidth}{!}{%
\begin{tikzpicture}[
    font=\small,
    panel/.style={draw=black!25, fill=black!2, rounded corners=3pt, inner sep=6pt},
    dot/.style={circle, fill=black!55, inner sep=1.45pt},
    vec/.style={-{Latex[length=2.3mm]}, line width=0.8pt},
    flow/.style={-{Latex[length=2.6mm]}, line width=1.15pt, draw=blue!60!black},
    latent/.style={-{Latex[length=2.6mm]}, line width=1.15pt, dashed, draw=red!70!black},
    stage/.style={font=\bfseries\small, text=black!80},
    note/.style={font=\scriptsize, align=center, text=black!70, text width=5.1cm},
    cell/.style={draw=black!25, minimum width=0.36cm, minimum height=0.36cm, inner sep=0pt}
]
% Panel A: combinatorial search.
\node[panel, minimum width=6.0cm, minimum height=4.1cm, anchor=south west] (a) at (0,4.9) {};
\node[stage, anchor=north west] at ($(a.north west)+(0.12,-0.12)$) {A. Enumerate graphs?};
\node[note, anchor=north] at ($(a.north)+(0,-0.55)$) {Score-based discovery starts with\\a super-exponential DAG family.};

\coordinate (g1) at (1.05,6.9);
\coordinate (g2) at (3.0,6.9);
\coordinate (g3) at (4.95,6.9);
\foreach \g/\flip in {g1/0,g2/1,g3/2} {
    \node[dot] (\g-a) at ($(\g)+(0,0.68)$) {};
    \node[dot] (\g-b) at ($(\g)+(-0.45,-0.12)$) {};
    \node[dot] (\g-c) at ($(\g)+(0.45,-0.12)$) {};
}
\draw[vec, black!60] (g1-a) -- (g1-b);
\draw[vec, black!60] (g1-a) -- (g1-c);
\draw[vec, black!60] (g2-b) -- (g2-a);
\draw[vec, black!60] (g2-b) -- (g2-c);
\draw[vec, black!60] (g3-c) -- (g3-a);
\draw[vec, black!60] (g3-a) -- (g3-b);
\node[font=\Large, text=black!55] at (5.7,6.9) {$\cdots$};
\node[note, anchor=south] at ($(a.south)+(0,0.25)$) {$\#\mathrm{DAGs}(d)$ grows too quickly to search directly};

% Panel B: local geometry.
\node[panel, minimum width=6.0cm, minimum height=4.1cm, anchor=south west] (b) at (7.25,4.9) {};
\node[stage, anchor=north west] at ($(b.north west)+(0.12,-0.12)$) {B. \textsc{Bridge}: measure local flows};
\node[note, anchor=north] at ($(b.north)+(0,-0.55)$) {Infinitesimal interventions become\\vector fields on the data manifold.};
\draw[fill=blue!6, draw=blue!25, line width=0.5pt]
    plot[smooth cycle] coordinates {(8.1,6.15) (8.85,7.55) (10.25,7.85) (11.95,7.1) (11.75,5.85) (9.75,5.55)};
\foreach \p in {(8.55,6.35),(9.1,7.05),(9.8,6.42),(10.35,7.25),(11.05,6.45),(11.48,6.9)} {
    \node[circle, fill=black!40, inner sep=0.9pt] at \p {};
}
\draw[flow] (8.75,6.0) to[out=25,in=205] node[above, font=\scriptsize, text=blue!60!black] {$v_i$} (10.35,6.6);
\draw[flow] (9.15,5.85) to[out=70,in=230] node[right, font=\scriptsize, text=blue!60!black] {$v_j$} (9.9,7.25);
\draw[latent] (10.35,6.6) to[out=20,in=165] node[above, font=\scriptsize, text=red!70!black] {$[v_i,v_j]_\perp$} (11.6,6.48);
\node[note, anchor=north] at ($(b.south)+(0,0.65)$) {Non-closing brackets flag failed visible closure};

% Panel C: pruned scoring.
\node[panel, minimum width=7.2cm, minimum height=4.5cm, anchor=south west] (c) at (3.0,-1.0) {};
\node[stage, anchor=north west] at ($(c.north west)+(0.12,-0.12)$) {C. Score only the \textsc{Bridge} mask};
\node[note, anchor=north] at ($(c.north)+(0,-0.55)$) {The geometry mask feeds\\constrained discovery runners.};
\foreach \r in {0,...,5} {
  \foreach \col in {0,...,5} {
    \node[cell, fill=black!4] at ($(4.0,1.65)+(\col*0.42,-\r*0.42)$) {};
  }
}
\foreach \col/\r in {0/1,1/3,2/0,2/4,3/5,4/2,5/1,5/4} {
    \node[cell, fill=green!55!black] at ($(4.0,1.65)+(\col*0.42,-\r*0.42)$) {};
}
\node[font=\scriptsize, text=black!65, align=center] at (5.05,-0.82) {candidate edge mask};
\coordinate (runnerhub) at (7.25,0.6);
\draw[vec, draw=black!55] (6.52,0.6) -- (runnerhub);
\node[draw=black!25, fill=green!7, rounded corners=2pt, align=center, font=\scriptsize, minimum width=1.75cm, minimum height=0.48cm] (ges) at (8.75,1.75) {GES/BIC};
\node[draw=black!25, fill=green!7, rounded corners=2pt, align=center, font=\scriptsize, minimum width=1.75cm, minimum height=0.48cm] (tces) at (8.75,0.95) {TCES/GES};
\node[draw=black!25, fill=blue!7, rounded corners=2pt, align=center, font=\scriptsize, minimum width=1.75cm, minimum height=0.48cm] (dcdi) at (8.75,0.15) {DCDI weights};
\node[draw=black!25, fill=black!2, rounded corners=2pt, align=center, font=\scriptsize, minimum width=1.85cm, minimum height=0.55cm] (family) at (8.75,-0.65) {selected\\DAG family};
\draw[vec, draw=black!45] (runnerhub) -- (ges.west);
\draw[vec, draw=black!45] (runnerhub) -- (tces.west);
\draw[vec, draw=black!45] (runnerhub) -- (dcdi.west);
\draw[vec, draw=black!45] (ges.south) -- (tces.north);
\draw[vec, draw=black!45] (tces.south) -- (dcdi.north);
\draw[vec, draw=black!45] (dcdi.south) -- (family.north);
\node[note, anchor=north] at ($(c.south)+(0,-0.55)$) {The hard global search becomes a smaller constrained search};

% Connecting arrows.
\draw[-{Latex[length=2.8mm]}, line width=1.1pt, draw=black!45] ($(a.south east)+(-0.35,0)$) -- ($(c.north west)+(0.35,0)$);
\draw[-{Latex[length=2.8mm]}, line width=1.1pt, draw=black!45] ($(b.south west)+(0.35,0)$) -- ($(c.north east)+(-0.35,0)$);
\end{tikzpicture}
}
\caption{\textsc{Bridge} pipeline for information-geometric causal discovery. \textsc{Bridge} first estimates local regime-response fields. Directed influence proposes candidate arrows, while Lie-bracket residuals flag failures of visible closure that require further discrimination among latent structure, target or domain mismatch, weak overlap, omitted coordinates, and estimation error. Downstream discovery is then run only inside the geometry-pruned family.}
\label{fig:geometry_discovery_pipeline}
\end{figure}
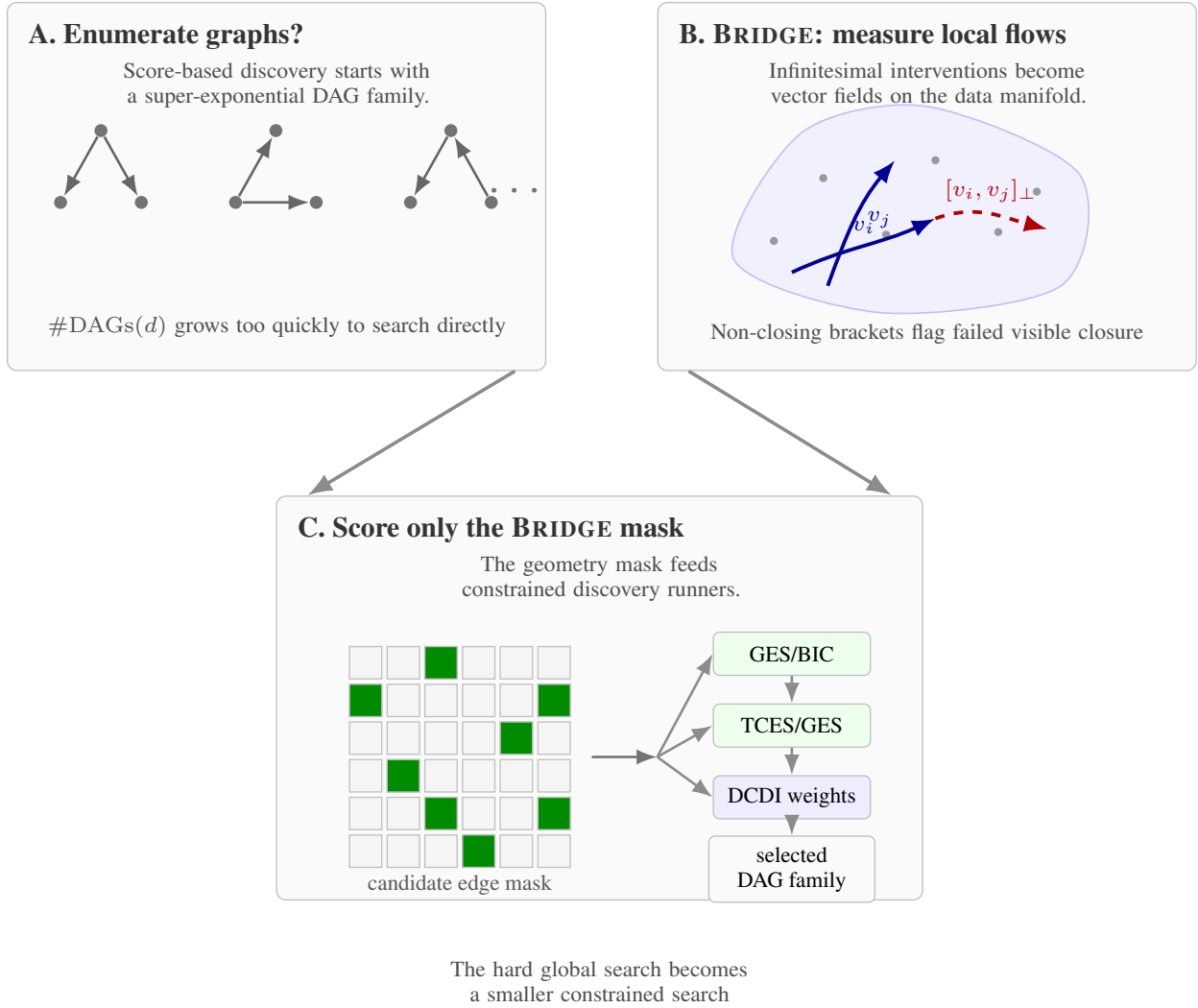

We introduce an  algorithmic family \textsc{Bridge} of causal discovery methods, short for \emph{Bracket Residuals for Interventional Discovery and Geometric Estimation}. Figure~\ref{fig:geometry_discovery_pipeline} illustrates the conceptual structure of the \textsc{Bridge} framework. 
Rather than beginning with a super-exponential family of labeled DAGs and asking which graph generated the data, \textsc{Bridge} asks where local causal information is compatible, where it fails to glue, and which visible directions witness the obstruction.  A \textsc{Bridge} method is any pipeline that estimates intervention-induced response fields, computes Lie/Frobenius residuals, and uses the resulting compatibility structure to constrain or regularize downstream discovery. It is deliberately not tied to one density estimator: normalizing flows, semiparametric influence-function estimators, Kernel--Stein energy models, or regularized Gaussian engines can all instantiate the interventional density layer when they supply calibrated Radon--Nikodym fields. A non-closing Lie bracket is therefore not treated as a final causal verdict, but as a computable witness of failed visible integrability: evidence for latent structure, regime mismatch, insufficient adjustment, or a missing coordinate in the causal description. The resulting discovery layer acts as a structural front end that proposes a smaller candidate family for downstream GES/BIC, TCES, or DCDI-style scoring.

This distinguishes the present work from our earlier sheaf-theoretic framework for decentralized causal discovery using Judo Calculus \citep{mahadevan2025decentralizedcausaldiscoveryusing}. In that setting, sheaf structure places existing discovery routines such as GES and DCDI inside a local-to-global consistency layer over regime covers: the central question is which local causal claims persist across contexts and which fail to glue. Here the consistency test is infinitesimal rather than sheaf-theoretic. Intervention-induced vector fields are tested for Lie/Frobenius integrability before score-based or differentiable discovery begins, so in favorable regimes the method does not merely stabilize local discoveries across a cover; it collapses the admissible search family that downstream algorithms must score.

Classical structural causal models define interventions by graph surgery or modified structural mechanisms. \textsc{Bridge} uses the distributions produced by those interventions as measure changes relative to a reference regime. Invertible flows or direct ratio estimators approximate the corresponding response fields; their calibration is statistical and does not require a categorical interpretation.

The distinction between association and intervention is already visible in the back-door criterion \citep{pearl-book}. A set $Z$ is a valid back-door adjustment set for the effect of $X$ on $Y$ when no element of $Z$ is a descendant of $X$ and $Z$ blocks every path from $X$ to $Y$ that enters $X$ through an incoming arrow. Under this condition the causal effect is identifiable by
\[
    P(Y\mid \mathrm{do}(X=x))=\sum_z P(Y\mid X=x,Z=z)P(Z=z),
\]
with the obvious integral replacing the sum for continuous $Z$. This identity is invoked only when its graphical assumptions hold; it is not derived from a Kan-extension duality.

Our paper explores the computational promise of this reformulation. The \textsc{Bridge} algorithm  uses local Lie/Frobenius geometry as a high-recall \emph{candidate-family reducer}: it screens directed arrows by local intervention-field influence and records non-closing brackets as latent-obstruction diagnostics. A downstream TCES-style Gaussian BIC scorer, GES-style search, or geometry-constrained differentiable discovery routine then scores only the pruned family. This hybrid keeps the global advantages of score-based and differentiable discovery while replacing exhaustive enumeration by a geometric local-to-global test. 

\subsection{\textsc{SKFM}: Causal Discovery as Lie-Algebraic Geometric Inference}

We also introduce and analyze a more ambitious direct method, \textsc{SKFM}. Whereas \textsc{Bridge} screens admissible arrows before downstream selection, \textsc{SKFM} asks how far graph extraction can proceed directly in the learned field space. It learns amortized response fields, estimates a spectral residual-footprint projector, and applies an ordered graph extractor with a triangular penalty. The experiments show that the ordering step, rather than field fitting, is the principal bottleneck outside simple motifs.

A common perspective in causal inference holds that the true DAG can be recovered without searching the super-exponential space of possible structures — provided one has access to a sufficient set of perfect interventions \citep{hauser2012characterization,maathuis2009estimating}. Indeed, under ideal conditions, such methods guarantee identifiability by narrowing the Markov equivalence class to a single DAG. However, this approach rests on assumptions that are rarely satisfied in practice: interventions must be perfectly targeted, fully observed, and free of confounding — and crucially, one must already know which variables to intervene on. In complex systems — such as gene regulatory networks, educational outcomes, or economic markets — interventions are often indirect (e.g., “increase school funding”), noisy, or partially observed. Moreover, latent confounders frequently violate the assumption that the true causal structure lies within the observed Markov equivalence class.

Rather than interpreting every residual as a hidden cause, \textsc{SKFM} factors the second moment of bracket residuals into directions of visible nonclosure. Under the explicit low-rank factorization and noncancellation assumptions of Proposition~\ref{latent_prop}, the rank of this footprint can equal the number of visible latent loading directions. Without those assumptions it is only a residual rank and need not equal the number, identity, or causal role of latent variables.

\subsection{Scope Under Latent Confounding}

FCI and RFCI permit arbitrary latent common causes represented through MAG/PAG structure; they do not require latents to be sparse, local, or mutually independent. Their assumptions and outputs differ from ours: they use conditional-independence information to identify invariant ancestral marks, whereas \textsc{Bridge} uses regime-response geometry as a screen. The factor assumptions used in our spectral proposition are therefore stronger assumptions introduced for a particular rank interpretation, not an advantage over FCI/RFCI.

\subsection{Theoretical Results on \textsc{Bridge} and \textsc{SKFM}}

The paper includes a scoped theoretical analysis. Proposition~\ref{geom-screen} proves retention of a prespecified population candidate set under field consistency and an influence margin; Corollary~\ref{downstream-selection} separately transfers the identification guarantee of a chosen downstream learner. Proposition~\ref{latent_prop} establishes spectral recovery of a visible residual-factor rank under an explicit low-rank factorization, finite moments, and an eigengap. Proposition~\ref{gradient_flow} supplies an ordered acyclicity certificate. These are screen-retention, residual-rank, and ordered-acyclicity results, not unrestricted recovery guarantees under arbitrary hidden confounding.

The rest of the paper constructs density-ratio engines and response fields, develops the \textsc{Bridge} screen and its retention guarantee, and presents \textsc{SKFM} with explicitly scoped rank and ordered-acyclicity results. The experiments distinguish field calibration, screen recall, ordering, and final graph selection; the final discussion lists the falsification tests required before a bracket residual is attributed to latent structure.

\section{Interventional Density and Response-Field Estimation}
\label{sec:ide}

We briefly summarize how normalizing flows, direct density-ratio estimators, semiparametric influence-function estimators, and
kernel-Stein energy models can supply the statistical interface used by \textsc{Bridge} \citep{rezende2015variational,dinh2016density,papamakarios2021normalizing,lipman2023flow,sugiyama2012density,bickel1993efficient,robins1994estimation,vanderlaan2011targeted,chernozhukov2018double,liu2016kernelized,chwialkowski2016kernel,gorham2017measuring}. Throughout, $(Y,X,Z)$ denote random variables over a common probability
space, with $P$ the observational law and $P^{(i)}$ the law obtained
under intervention $i$.

% ======================================================================
\subsection{Density Engines as RN-Ratio Providers}

\textsc{Bridge} requires an estimator of the Radon--Nikodym ratio between an
interventional law $P^{(i)}$ and the observational law $P$:
\[
    \rho_i(\omega)=\frac{dP^{(i)}}{dP}(\omega),\qquad
    \ell_i(\omega)=\log \rho_i(\omega).
\]
Thus any model that supplies observational and interventional log-densities can
serve as a response-field engine:
\[
    \ell_i(\omega)
    \approx
    \log p_{\mathrm{do}}(\omega\mid i)-\log p_{\mathrm{obs}}(\omega),
\]
with additive constants cancelling in RN-ratio identities. Normalizing flows provide tractable log-likelihoods or continuous transport velocities; density-ratio methods estimate $\rho_i$ directly without separately estimating both densities; semiparametric influence-function and doubly robust estimators supply orthogonal estimating equations and uncertainty controls for interventional functionals; and Kernel--Stein energy models use score identities to fit or compare distributions even when normalizing constants are unavailable. These estimator families differ in statistical assumptions and uncertainty
guarantees, but they enter the present discovery layer through this same
interface. The experiments below use regularized Gaussian density engines
because they make the first tractability claim transparent and cheap to
reproduce; replacing them by flows or energy models changes only how
$\ell_i$ and its derivatives are estimated.

When a valid back-door set $Z$ is known, one may additionally regularize an
interventional engine toward the classical identity
\[
    P(y\mid\mathrm{do}(x))=\int P(y\mid x,z)P(z)\,dz
\]
through an adjustment-calibration loss
\[
\mathcal{L}_{\mathrm{adj}}
=
\mathbb{E}_{x}\!\left[
D\!\left(
q_\theta(y\mid\mathrm{do}(x)),\,
\int p_\phi(y\mid x,z)p_\psi(z)\,dz
\right)
\right],
\]
where $D$ is a statistical divergence. This optional term is justified by the
back-door identity and only when $Z$ is a valid adjustment set.

\subsection{From Density Ratios to Discovery Operators}

The same RN ratios also generate the local objects used by the discovery
algorithm. For candidate intervention $i$, $\rho_i$ defines asymmetric
influence scores by measuring how the $i$-interventional change of measure
transports mass in each visible coordinate. Its score or flow velocity defines
a local causal vector field,
\[
    v_i(z)\approx \nabla_z \ell_i(z)
    \quad\text{or}\quad
    v_i(z,t)=\frac{\partial}{\partial t}g_i(z,t),
\]
depending on whether one uses a static density-ratio model or a continuous
flow. Differentiating these fields gives the local Jacobians needed for Lie
brackets $[v_i,v_j]$ and Frobenius residuals. Thus the density engine is not a
separate causal-effect module; it is the numerical source of directed
influence, latent-curvature diagnostics, and the geometry-pruned candidate mask.

\section{\textsc{Bridge} Response Fields and Differential Geometry}
\label{sec:differential-geometric}

When these operators are parameterized using Continuous Normalizing Flows (CNFs), the framework takes on a clear differential-geometric meaning. The space of smooth, mutually absolute-continuous probability distributions defines a statistical manifold $\M$.

A differentiable path of regime laws has a tangent represented by a mean-zero score in the statistical tangent space, while a transport model supplies a vector field on sample space. We use one of two explicitly distinct objects:
\[
 s_e(z)=\nabla_z\log\rho_e(z)
 \qquad\text{or}\qquad
 v_e(z,t)=\partial_t g_e(z,t).
\]
The first is a spatial score-difference field; the second is a learned transport velocity satisfying a continuity equation for the modeled path. They are related only under additional choices of dynamics and must not be equated with the scalar $\rho_e-1$. Integrating $v_e$ produces the learned transport path; it is not asserted to be a unique geodesic.

\subsection{\v{C}encov's Invariance and the Metric Energy Cost}

To evaluate the informational energy cost of a causal transformation along a vector path, the statistical manifold must be equipped with a metric tensor $g_{ij}$. This choice is uniquely locked by \v{C}encov's Theorem \citep{chentsov1982}. 

\begin{theorem}[\v{C}encov]
\label{cencov}
Up to an overall scale, the Fisher information metric is the unique Riemannian metric on finite-dimensional statistical models invariant under congruent Markov morphisms, subject to the theorem's regularity assumptions.
\end{theorem}

\v{C}encov's theorem motivates using the Fisher metric for a coordinate-invariant statistical energy cost:
\begin{equation}
    E = \int \langle v, v \rangle_{\text{Fisher}} \, dt
    \label{eq:energy}
\end{equation}
This energy is coordinate-free under the Markov morphisms covered by \v{C}encov's invariance theorem. In the population limit, and when the neural path class contains the relevant Fisher geodesic, minimizing Equation~\ref{eq:energy} selects the minimum-energy representative among observational-to-interventional transports. In finite samples and restricted architectures, it should be read as a geometrically natural regularizer rather than as a certificate that training has found the exact geodesic.

\subsection{Neural Flow Training Algorithm and Loss Formulations}
To estimate the regime changes, we employ an invertible neural architecture---a family of continuous normalizing flows (CNFs) or standard autoregressive flows---to parameterize the probability laws.

\subsection{Architectural Parameterization}
Let $g_\theta: \R^d \to \R^d$ be a bijective, differentiable neural mapping parameterized by weights $\theta$. We instantiate two independent flow networks to capture the respective geometries of the observational and interventional manifolds:
\begin{align}
    p_{\text{obs}}(z; \theta_{\text{obs}}) &= p_0(g_{\theta_{\text{obs}}}^{-1}(z)) \left| \det \nabla_z g_{\theta_{\text{obs}}}^{-1}(z) \right| \\
    p_{\text{do}}(z; \theta_{\text{do}}) &= p_0(g_{\theta_{\text{do}}}^{-1}(z)) \left| \det \nabla_z g_{\theta_{\text{do}}}^{-1}(z) \right|
\end{align}
where $p_0$ is a standard base distribution (e.g., an isotropic Gaussian). The continuous Radon--Nikodym derivative field is evaluated pointwise as the change-of-variables ratio:
\begin{equation}
    \rho_\theta(z) = \frac{p_{\text{do}}(z; \theta_{\text{do}})}{p_{\text{obs}}(z; \theta_{\text{obs}})}
\end{equation}

\subsection{The Calibration Objective}
Rather than relying only on independent maximum-likelihood fits, we regularize the ratio by the change-of-measure calibration identity: reweighting observational samples by $\rho_\theta$ should reproduce interventional moments within the chosen function class.

To enforce this globally without hand-picking individual test functions $f$, we construct a dual-network adversarial loss or use a Maximum Mean Discrepancy (MMD) kernel objective. The empirical calibration loss $\mathcal{L}_{\text{cal}}$ over a reproducing kernel Hilbert space $\mathcal{H}$ with kernel $k(z, z')$ is formulated as:
\begin{align}
    \mathcal{L}_{\text{cal}}(\theta_{\text{obs}}, \theta_{\text{do}}) = &\frac{1}{N^2}\sum_{i=1}^N\sum_{j=1}^N k(z_i^{\text{do}}, z_j^{\text{do}}) \nonumber \\
    &- \frac{2}{NM}\sum_{i=1}^N\sum_{m=1}^M \rho_\theta(z_m^{\text{obs}})\, k(z_i^{\text{do}}, z_m^{\text{obs}}) \nonumber \\
    &+ \frac{1}{M^2}\sum_{m=1}^M\sum_{n=1}^M \rho_\theta(z_m^{\text{obs}})\rho_\theta(z_n^{\text{obs}})\, k(z_m^{\text{obs}}, z_n^{\text{obs}})
\end{align}
where $\{z_i^{\text{do}}\}_{i=1}^N \sim \ddo$ and $\{z_m^{\text{obs}}\}_{m=1}^M \sim \dobs$. Minimizing $\mathcal{L}_{\text{cal}}$ matches the reweighted empirical distribution to the regime sample in the chosen RKHS. Finite-sample calibration in one kernel class is not exact equality of the underlying laws.

\subsection{Information-Geometric Path Regularization}
To bias the network toward information-efficient paths in the Fisher geometry singled out by \v{C}encov's Theorem, we add a kinetic path penalty. For a continuous flow governed by the vector field $v_\theta(z, t) = \frac{\partial}{\partial t} g_\theta(z, t)$, the Fisher energy regularization term $\mathcal{L}_{\text{Fisher}}$ is defined as:
\begin{equation}
    \mathcal{L}_{\text{Fisher}}(\theta) = \int_0^1 \E_{p(z, t)} \left[ \| v_\theta(z, t) \|^2_2 \right] dt
\end{equation}
The total optimized loss function combines both terms using a structural scaling hyperparameter $\lambda$:
\begin{equation}
    \mathcal{L}_{\text{total}} = \mathcal{L}_{\text{cal}}(\theta_{\text{obs}}, \theta_{\text{do}}) + \lambda \mathcal{L}_{\text{Fisher}}(\theta_{\text{do}})
\end{equation}

\subsection{A Neural-Flow \textsc{Bridge} Training Procedure}
The complete optimization routine is detailed in Algorithm \ref{alg:flow_training}. This is one \textsc{Bridge} instance: a neural-flow density engine that estimates calibrated Radon--Nikodym fields before the Lie/Frobenius screening stage. Other \textsc{Bridge} variants can replace this training block with any interventional density estimator that provides the log-ratio fields and derivatives needed for directed influence scores and bracket residuals.

\begin{algorithm}[H]
\caption{\textsc{Bridge}-Flow: Neural Calibration of Interventional Density Fields}\label{alg:flow_training}
\begin{algorithmic}[1]
\REQUIRE Observational dataset $\mathcal{D}_{\text{obs}}$, interventional dataset $\mathcal{D}_{\text{do}}$, learning rate $\eta$, hyperparameter $\lambda$.
\ENSURE Calibrated causal density field parameters $\theta_{\text{obs}}, \theta_{\text{do}}$.
\STATE Initialize flow network weights $\theta_{\text{obs}}$ and $\theta_{\text{do}}$ randomly.
\WHILE{not converged}
    \STATE Sample mini-batch $\{z_m^{\text{obs}}\}_{m=1}^M$ from $\mathcal{D}_{\text{obs}}$ and $\{z_i^{\text{do}}\}_{i=1}^N$ from $\mathcal{D}_{\text{do}}$.
    \STATE Compute log-likelihood profiles $p_{\text{obs}}(z_m^{\text{obs}})$ and $p_{\text{do}}(z_i^{\text{do}})$.
    \STATE Evaluate localized Radon--Nikodym weight vectors: $\rho_\theta(z_m^{\text{obs}}) \leftarrow p_{\text{do}}(z_m^{\text{obs}})/p_{\text{obs}}(z_m^{\text{obs}})$.
    \STATE Compute calibration error $\mathcal{L}_{\text{cal}}$ using the kernelized MMD matrix framework.
    \STATE Integrate continuous velocity fields to calculate the Fisher kinetic cost $\mathcal{L}_{\text{Fisher}}$.
    \STATE Update total loss: $\mathcal{L}_{\text{total}} \leftarrow \mathcal{L}_{\text{cal}} + \lambda \mathcal{L}_{\text{Fisher}}$.
    \STATE Stochastic gradient step: $\theta_{\text{obs}} \leftarrow \theta_{\text{obs}} - \eta \nabla_{\theta_{\text{obs}}} \mathcal{L}_{\text{total}}$.
    \STATE Stochastic gradient step: $\theta_{\text{do}} \leftarrow \theta_{\text{do}} - \eta \nabla_{\theta_{\text{do}}} \mathcal{L}_{\text{total}}$.
\ENDWHILE
\RETURN $\theta_{\text{obs}}, \theta_{\text{do}}$
\end{algorithmic}
\end{algorithm}

\section{Causal Discovery via Asymmetric Flow Fields and Lie Brackets}
\label{sec:causal-discovery}
Continuous causal discovery methods such as NOTEARS \citep{notears} enforce acyclicity through a global matrix constraint. \textsc{Bridge} instead uses localized response operators to screen candidate arrows before an acyclicity-aware learner is applied. This changes the size of the family presented to the downstream method; it does not by itself enforce acyclicity.

\subsection{Asymmetric Edge Scoring via Residual Variance Reduction}
Under a correctly targeted intervention and standard SCM semantics, manipulating an ancestor can change a descendant's distribution, whereas manipulating a descendant does not change an ancestor's structural mechanism. Distributional equality can nevertheless occur through cancellation, and indirect paths or latent common causes can also create large pairwise responses. We therefore use directional response only as a screening statistic.

In our geometric framework, this ancestral asymmetry maps directly onto the behavior of the calibrated Radon--Nikodym derivative field $\rho_i(z)$. Instead of dynamically swapping discrete graph edges, we evaluate the directed causal edge score $S(X_i \to X_j)$ by measuring the localized predictive variance reduction under the interventional change of measure:

\begin{equation}
    S(X_i \to X_j) = \log \left( \frac{\text{Var}_{\dobs}(X_j)}{\text{Var}_{\dobs}(X_j \cdot \rho_i(Z))} \right)
\end{equation}

This heuristic behaves as a directional information filter in favorable regimes:
\begin{itemize}
    \item A large $S(X_i\to X_j)$ indicates that the $i$-regime reweighting substantially changes the variance profile of $X_j$ and merits retaining $i\to j$ as a candidate.
    \item A small score is evidence against a detectable effect at the available intervention strength and sample size, not proof of the absence of a direct arrow.
\end{itemize}

\subsection{Lie-Bracket Closure as a Diagnostic}
To reduce the burden placed on global matrix regularizers, the framework evaluates the commutativity of the learned interventional vector fields before downstream graph scoring. Let $v_i, v_j \in T\M$ be the smooth tangent vector fields corresponding to the continuous normalizing flow updates for interventions on $X_i$ and $X_j$, respectively.

The local order-sensitivity of the intervention fields is measured by the \textbf{Lie Bracket}, defined as:
\begin{equation}
    [v_i, v_j] = \nabla v_j \cdot v_i - \nabla v_i \cdot v_j
\end{equation}
For machine-learning readers, a useful intuition is to view $v_i$ as a local current associated with regime $i$ and $v_j$ as a response field transported through that current. Arnold and Khesin describe the bracket as a ``fisherman derivative'' \citep{arnold1999topological}. The bracket measures noncommutativity of these \emph{learned transport fields}, not noncommutativity of fixed do-operations. Its value depends on how finite regime changes are embedded into smooth flows and extended away from observed points. A nonzero residual therefore diagnoses order sensitivity of the representation and failure of visible closure, not a uniquely causal mechanism.

\begin{figure}[t]
\centering
\begin{tikzpicture}[
    >=Latex,
    font=\footnotesize,
    flow/.style={blue!55!black, thick, ->},
    response/.style={green!45!black, very thick, ->},
    residual/.style={red!70!black, very thick, ->},
    dot/.style={circle, fill=black!70, inner sep=1.4pt},
    panel/.style={draw=black!18, rounded corners=3pt, fill=black!2},
    labelbox/.style={fill=white, fill opacity=0.88, text opacity=1,
        inner sep=1.5pt, rounded corners=1.5pt},
    leader/.style={black!45, thick}
]
    \begin{scope}[xshift=-0.4cm]
        \fill[blue!8, rounded corners=4pt] (-0.25,-0.85) rectangle (6.15,2.9);
        \node[anchor=west, text=black!65] at (0,2.55) {observational statistical manifold};

        \draw[blue!45!black, thick]
            (0,0.15) .. controls (1.1,0.65) and (2.1,-0.15) .. (3.25,0.35)
            .. controls (4.25,0.8) and (5.0,0.2) .. (5.85,0.55);
        \draw[blue!45!black, thick]
            (0,-0.45) .. controls (1.1,-0.05) and (2.1,-0.75) .. (3.25,-0.25)
            .. controls (4.25,0.2) and (5.0,-0.4) .. (5.85,-0.05);
        \draw[blue!45!black, thick]
            (0,0.85) .. controls (1.15,1.25) and (2.15,0.55) .. (3.35,1.05)
            .. controls (4.15,1.35) and (5.15,0.85) .. (5.85,1.2);

        \node[dot] (p) at (2.35,0.18) {};
        \node[labelbox, above left=1pt and 2pt of p] {$P$};
        \draw[flow] (p) -- ++(1.15,0.25)
            node[labelbox, pos=0.66, above, sloped] {$v_i$};
        \draw[response] (p) -- ++(0.55,1.05)
            node[labelbox, pos=0.58, left] {$v_j$};
        \draw[response, opacity=0.55] (3.45,0.42) -- ++(0.48,0.82);
        \node[labelbox, anchor=west, align=left, text=green!45!black] at (4.52,1.23)
            {transported\\$v_j$};
        \draw[leader, green!45!black] (4.37,1.11) -- (3.88,0.98);
        \draw[residual] (3.45,0.42) -- ++(0.95,-0.5)
            node[labelbox, pos=0.83, below, sloped] {drag $[v_i,v_j]_{\perp}$};

        \draw[leader] (1.74,1.86) -- (2.16,1.26) -- (2.35,0.18);
        \node[labelbox, align=center, text=black!65] at (0.82,1.88)
            {observer moves\\with $v_i$};
    \end{scope}

    \begin{scope}[xshift=6.7cm]
        \node[panel, minimum width=3.4cm, minimum height=1.35cm, anchor=north west] (a) at (0,2.75) {};
        \node[labelbox, anchor=west] at (0.18,2.48) {$do(X_i)$ then $do(X_j)$};
        \node[dot] (a0) at (0.55,1.58) {};
        \draw[flow] (a0) -- ++(0.8,0.18);
        \draw[response] (1.35,1.76) -- ++(0.58,0.43);
        \node[dot, fill=green!45!black] at (1.93,2.19) {};

        \node[panel, minimum width=3.4cm, minimum height=1.35cm, anchor=north west] (b) at (0,1.0) {};
        \node[labelbox, anchor=west] at (0.18,0.72) {$do(X_j)$ then $do(X_i)$};
        \node[dot] (b0) at (0.55,-0.03) {};
        \draw[response] (b0) -- ++(0.58,0.43);
        \draw[flow] (1.13,0.40) -- ++(0.8,0.18);
        \node[dot, fill=red!70!black] at (1.93,0.58) {};

        \draw[residual] (2.55,2.02) -- (2.55,0.75);
        \node[labelbox, anchor=west, align=left, text=red!70!black] at (3,1.42)
            {endpoint mismatch\\nonzero bracket};
        \node[labelbox, align=center, text=black!70] at (1.7,-0.95)
            {visible order-sensitivity\\flags latent obstruction};
    \end{scope}
\end{tikzpicture}
\caption{Fisherman-derivative intuition for the learned response fields. The residual $[v_i,v_j]_{\perp}$ compares two learned flow orders. It concerns the chosen transport representation, not the algebra of fixed do-assignments; \textsc{Bridge} uses it as a nonspecific visible-closure diagnostic before downstream search.}
\label{fig:fisherman_bridge_intuition}
\end{figure}
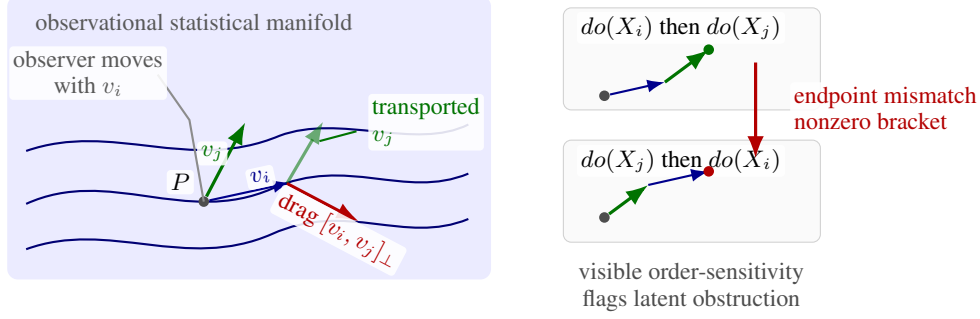

The Lie bracket therefore measures local order sensitivity of the modeled regime transports.
\begin{enumerate}
    \item \textbf{Closure:} If $[v_i,v_j]$ lies in the modeled visible span, the pair is locally compatible with that response-field distribution.
    \item \textbf{Nonclosure:} An orthogonal residual can reflect latent forcing, target mismatch, domain shift, weak overlap, insufficient adjustment, field-estimation error, or an omitted nonlinear coordinate.
\end{enumerate}

Frobenius's theorem \citep{lee2012smooth} makes involutive closure the condition for local integrability of the distribution spanned by smooth fields. We use this as a diagnostic principle. Failure to close flags an incomplete or misspecified visible description; it does not prove global acyclicity or latent confounding. \textsc{Bridge} uses closure as a local screen before acyclicity-aware or latent-aware discovery is applied.

\paragraph{\textsc{Bridge} running time.}
Let $d$ be the number of observed variables, $n$ the number of evaluation samples, $m$ the number of directed arrows retained by the geometric mask, $r_j$ the number of candidate parents retained for target $j$, $r=\max_j r_j$, and $k$ the maximum indegree allowed by the downstream scorer. The \textsc{Bridge} geometric layer is best separated into a field-estimation cost $C_{\mathrm{field}}$ and a bracket/screening cost. In the current kernel/local-linear teacher implementation, estimating all ordered pairwise response fields costs
\[
  C_{\mathrm{field}}
  =
  O\!\left(d^2 n (n a+a^3)\right)
\]
with $a$ adjustment coordinates. If a single anchor intervention field $v_i$ is fitted per node and then reused for all target coordinates, the leading field-estimation term is instead $O(d n(na+a^3))$, although the full directed screen still evaluates $d(d-1)$ coordinate influences. A continuous-normalizing-flow RN-ratio variant adds an inference term $O(d n K)$ for $K$ ODE steps per field; a conditional-flow-matching or other amortized velocity model replaces this by $O(d n P_\Theta)$ field evaluation for a network with $P_\Theta$ parameters. Given kernel-estimated fields, the released NumPy bracket screen uses smoothed finite differences and has the higher teacher-style cost
\[
  O(d^3 n^2) + O(d^2 n s^3),
\]
where $s=2$ for pair-span projection and $s=d$ for all-field projection. In an amortized autodiff implementation, however, all field Jacobians can be obtained with Jacobian-vector products at cost about $O(d n P_\Theta)$, and assembling all pair brackets costs $O(d^3 n)$ for the resulting matrix-vector products. Influence scoring and mask construction are lower-order, $O(d^2 n+d^2\log d)$, and optional exact counting of acyclic subgraphs costs $O(2^m(d+m))$ but is capped in the released code.

The downstream term depends on which scorer is plugged into the mask. For the exact TCES/BIC scorer used in the small controlled experiments, the candidate parent sets for node $j$ are
\[
  P_j=\sum_{\ell=0}^{\min(k,r_j)} \binom{r_j}{\ell},
\]
so local Gaussian BIC scores cost $\sum_j P_j\,O(nk^2+k^3)$ after caching, while exact graph enumeration costs $\prod_j P_j$ acyclicity checks and optional TCES/sheaf penalties. This is still exponential in $d$ in the worst case, but the exponent is now controlled by the retained mask and bounded indegree rather than by the full $d(d-1)$ arrow set. For the ordered local BIC diagnostic in the ten-node random-DAG experiment, the graph product disappears: one scores $\sum_j P_j=O(d r^k)$ parent sets when $r_j\le r$ and $k$ is fixed. For the greedy GES-style pruned search, each hill-climb step scans additions, removals, and reversals only inside the mask, giving roughly $O(L\,m\,(C_{\mathrm{local}}+d+m))$ for $L$ accepted or attempted search rounds, with cached local scores and cycle checks. Thus \textsc{Bridge} is best understood as a polynomial geometric front end followed by a downstream search whose residual combinatorics are functions of $(m,r_j,k)$, not of the unrestricted labeled-DAG family. This is why the ten-node experiment can run quickly after screening: the retained masks have about $18$--$22$ arrows and the ordered BIC scorer evaluates only bounded-size parent sets.
When the mask itself has at most $k$ candidate parents per node, this simplifies to $O(d\,2^k)$ parent-set evaluations; otherwise $O(d r^k)$ is the more conservative bound. In summary, the amortized \textsc{Bridge}-hybrid path has the schematic complexity
\[
  \widetilde O\!\left(d^3 n + d r^k\right),
\]
after field learning, with the bounded-mask special case $\widetilde O(d^3 n+d\,2^k)$. This should be contrasted with unconstrained score search over a super-exponential DAG family, or over Markov equivalence classes whose count remains a constant fraction of the number of DAGs \citep{gillispie2002size,schmid2024number}. The geometric layer therefore does not eliminate all combinatorics; it collapses the global exponential into a bounded local parent-set search inside a continuous, polynomial-time screen.

The following result concerns screen retention only. Downstream identification is separated into a corollary because it depends entirely on the estimand and assumptions of the chosen learner.

\begin{proposition}[Geometry-Screen Retention]
\label{geom-screen}
Let $v_i$ denote the population intervention vector field for intervention $i$, and define the population influence score
\[
    I_{ij}=\left(\mathbb{E}_{P}\big[(v_i)_j^2\big]\right)^{1/2}.
\]
Let $R_{ij}$ be the population closure residual obtained by projecting $[v_i,v_j]$ onto the chosen visible span. Suppose: (i) each arrow in a prespecified population target set $E^\star$ lies among the top-$k$ population influences into its endpoint with margin $\gamma>0$; (ii) the estimated fields $\widehat v_i$ converge uniformly to $v_i$ in $L^2(P)$; and (iii) the estimated bracket residuals $\widehat R_{ij}$ are uniformly consistent. Then the screen retains every member of $E^\star$ with probability tending to one. If two prespecified diagnostic classes satisfy $R_{ij}>\tau+\gamma$ and $R_{ij}<\tau-\gamma$, respectively, thresholding $\widehat R_{ij}$ at $\tau$ consistently separates those classes. The proposition does not identify either class with latent confounding without a separate specificity argument.
\end{proposition}

\begin{proof}
Let $\widehat I_{ij}=(\mathbb{E}_n[(\widehat v_i)_j^2])^{1/2}$ denote the empirical influence score. Under uniform $L^2(P)$ convergence of $\widehat v_i$ to $v_i$, bounded second moments, and the usual Glivenko--Cantelli condition for the finite collection of squared coordinate fields used by the screen, $\max_{i,j}|\widehat I_{ij}-I_{ij}|\to 0$ in probability. Choose $\epsilon<\gamma/2$. With probability tending to one, the empirical scores are all within $\epsilon$ of their population values. On this event, the population top-$k$ margin cannot be reversed by estimation error, so every member of $E^\star$ remains top-$k$ empirically.

For the Frobenius residuals, assume the intervention fields lie in a bounded $C^1$ class on the compact region carrying the relevant probability mass, or more generally that the local Jacobian and projection estimators are uniformly consistent in the norm used to form $R_{ij}$. The Lie bracket map
\[
    (v_i,v_j)\mapsto [v_i,v_j]=D v_j\,v_i-D v_i\,v_j
\]
is continuous in the corresponding $C^1$ norm. Therefore the estimated brackets and their visible-span projections converge uniformly, giving $\max_{i,j}|\widehat R_{ij}-R_{ij}|\to 0$ in probability. If the prespecified diagnostic classes are separated by a margin around $\tau$, then for all sufficiently large samples the empirical threshold at $\tau$ has the same classification as the population threshold.

Combining the two events, the retained candidate family contains $E^\star$ with probability tending to one. This is a screening guarantee, not a finite-sample orientation or latent-identification theorem.
\end{proof}

\begin{corollary}[Conditional downstream selection]
\label{downstream-selection}
Suppose the hypotheses of Proposition~\ref{geom-screen} hold and a downstream learner is consistent for an estimand $\mathcal G^\star$ whenever its candidate family contains the features required to represent $\mathcal G^\star$. Then the two-stage procedure is consistent for $\mathcal G^\star$. Here $\mathcal G^\star$ may be a PAG or interventional equivalence-class feature set for a latent-aware learner, or a unique DAG only under assumptions that identify one.
\end{corollary}

\begin{proof}
Screen retention implies that the restricted family contains the representation required by the downstream estimand with probability tending to one. Conditional consistency of the downstream learner then gives the result. No additional identification is supplied by the screen itself.
\end{proof}

\section{Bracket Residuals as Nonspecific Visible-Closure Diagnostics}
\label{sec:latent-obstruction}
Let $\mathcal D_V(z)=\operatorname{span}\{v_1(z),\ldots,v_q(z)\}$ be the distribution spanned by the learned visible response fields, and let $\Pi_V(z)$ denote projection onto that span. We define the pairwise closure residual
\[
  r_{ij}(z)=\bigl(I-\Pi_V(z)\bigr)[v_i,v_j](z),
  \qquad
  R_{ij}^2=\mathbb E_{P^{(0)}}\|r_{ij}(Z)\|^2.
\]
The residual tests whether the fitted field family closes locally. It is not a curvature tensor, a topological invariant, or an identification functional for a hidden common cause.

\subsection{Why Latent Variables Can Contribute}
If an unobserved factor drives several measured mechanisms, projection onto observed coordinates can omit directions needed to express brackets of the induced response fields. This provides a plausible source of a nonzero $r_{ij}$. Under the restrictive additive-factor assumptions in Proposition~\ref{latent_prop}, the second moment of these residuals has a recoverable low-rank visible footprint.

\subsection{Why the Diagnostic Is Not Specific}
The same residual can be produced by a mislabeled or unknown intervention target, a domain change that is not a shared-SCM intervention, weak overlap and ratio instability, failure to condition on an available adjustment coordinate, misspecified or poorly estimated fields, or an insufficient coordinate dictionary. Conversely, a latent common cause need not produce a detectable residual: effects can cancel, the visible latent loadings can be rank deficient, or the chosen fields can close despite the latent variable. Revealing an adjustment variable may reduce a residual in an example, but no general theorem says that valid adjustment forces the bracket to vanish.

\subsection{Required Falsification Checks}
Before calling a pair latent-sensitive, we compare the bracket score to simpler diagnostics: the direct observational--interventional distribution gap, conditional-invariance tests, density-ratio overlap and effective sample size, field-calibration error, random or permuted tangent controls, and sensitivity to adding measured adjustment coordinates. For unknown soft targets, $\Psi$-FCI is a matched causal-discovery baseline; for context-indexed domains, JCI is the matched conceptual baseline. These checks do not make the bracket infallible, but they test whether it adds information beyond ordinary regime discrepancy and whether the latent interpretation survives obvious alternatives.

\subsection{Operational Role}
The current paper uses $R_{ij}$ to rescue or penalize candidates inside a high-recall mask and to summarize a residual subspace. Final causal interpretation comes from a regime-appropriate downstream learner and substantive assumptions. Thus ``latent-sensitive'' means only that latent structure is among the hypotheses not rejected by the diagnostic battery.

\section{\textsc{SKFM}: Spectral Kernel Flow Matching} 
\label{sec:skfm}

We propose \textbf{Spectral Kernel Flow Matching (SKFM)} as an amortized extension of the two-stage \textsc{Bridge} pipeline. \textsc{SKFM} replaces kernel-estimated fields with conditional-flow-matched fields, summarizes the Lie-bracket residual Gram matrix by a spectral projector, and imposes a soft triangular solvable-Lie penalty relative to a supplied or learned order. The chain and selected motif experiments show what is possible when this ordering assumption is correct. The non-chain and random-DAG ablations show the present limitation: field learning remains well calibrated, but Step~9, which discretizes continuous geometry into an adjacency, is sensitive to the learned order and extraction rule. \textsc{Bridge} with a downstream scorer is therefore the primary method; direct \textsc{SKFM} extraction is a scoped diagnostic endpoint.

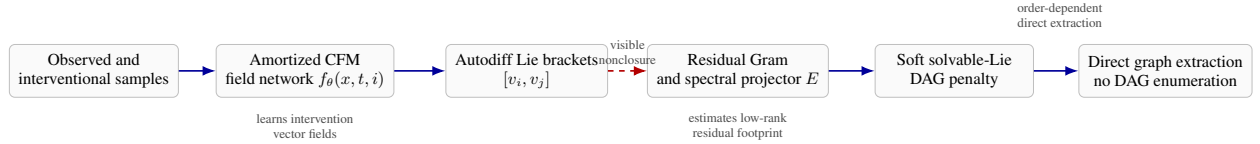
\begin{figure}[h]
\centering
\resizebox{\textwidth}{!}{%
\begin{tikzpicture}[
    font=\small,
    box/.style={draw=black!25, fill=black!2, rounded corners=3pt, align=center, minimum width=3.0cm, minimum height=1.0cm, inner sep=5pt},
    flow/.style={-{Latex[length=2.5mm]}, line width=0.9pt, draw=blue!60!black},
    latent/.style={-{Latex[length=2.5mm]}, line width=0.9pt, dashed, draw=red!70!black},
    note/.style={font=\scriptsize, align=center, text=black!70}
]
\node[box] (data) at (0,0) {Observed and\\interventional samples};
\node[box] (cfm) at (4.0,0) {Amortized CFM\\field network $f_\theta(x,t,i)$};
\node[box] (bracket) at (8.2,0) {Autodiff Lie brackets\\$[v_i,v_j]$};
\node[box] (spectrum) at (12.2,0) {Residual Gram\\and spectral projector $E$};
\node[box] (dag) at (16.3,0) {Soft solvable-Lie\\DAG penalty};
\node[box] (graph) at (20.3,0) {Direct graph extraction\\no DAG enumeration};

\draw[flow] (data) -- (cfm);
\draw[flow] (cfm) -- (bracket);
\draw[latent] (bracket) -- node[above, note] {visible\\nonclosure} (spectrum);
\draw[flow] (spectrum) -- (dag);
\draw[flow] (dag) -- (graph);

\node[note] at (4.0,-1.05) {learns intervention\\vector fields};
\node[note] at (12.2,-1.05) {estimates low-rank\\residual footprint};
\node[note] at (18.3,1.05) {order-dependent\\direct extraction};
\end{tikzpicture}
}
\caption{\textsc{SKFM} as a direct counterpart to the main screening pipeline. It learns amortized response fields, factors their closure residuals spectrally, and applies an order-dependent graph extractor. The spectral coordinates summarize visible nonclosure and are not identified latent variables without Proposition~\ref{latent_prop}'s assumptions.}
\label{fig:skfm_pipeline}
\end{figure}

\subsection{From Screening to Direct Extraction}

However, the \textsc{Bridge} implementation evaluated in Section~\ref{sec:experiments} decouples these into:
\begin{itemize}
    \item \textbf{Stage 1:} Geometric screening via Frobenius anisotropy to produce a candidate edge mask $\mathcal{E}_{\text{cand}}$
    \item \textbf{Stage 2:} Discrete score-based search (GES/BIC/TCES) over the restricted DAG family $\mathcal{G}(\mathcal{E}_{\text{cand}})$
\end{itemize}

This separation motivates, but does not guarantee, a direct extractor:
\begin{enumerate}
    \item \textit{Discontinuity:} A hard residual threshold discards graded information about visible nonclosure.
    \item \textit{Residual structure:} Spectral summaries may reveal a stable low-rank footprint even when they do not identify latent variables.
    \item \textit{Combinatorial Remnant:} GES remains super-exponential in the worst case within the mask.
\end{enumerate}

\subsection{Proposed Algorithm: Spectral Kernel Flow Matching}

We integrate these stages by using triangular closure structure as an acyclicity regularizer and the orthogonal bracket component as a residual-footprint statistic. These are modeling choices; a general equivalence between DAGs and solvable algebras is not claimed.

Algorithm~\ref{alg:skfm} gives the resulting procedure. It first learns amortized response fields by conditional flow matching, then uses autodifferentiated Lie brackets to form a residual Gram matrix, extracts a low-rank footprint projector, and finally imposes a soft triangular penalty before reading off a selected visible graph. The learning and spectral steps are useful diagnostics; the fragile step is Step~9, which turns influence scores into a hard ordered adjacency. Robust order-free extraction remains open.

\begin{algorithm}
\caption{Spectral Kernel Flow Matching (SKFM-KT)}
\label{alg:skfm}
\begin{algorithmic}[1]
\REQUIRE Observational and known-target interventional datasets $\mathcal{D}_{\text{obs}}, \{\mathcal{D}_{\text{do}(i)}\}$
\STATE Initialize amortized flow network $\bm{f}_\Theta$ with embeddings $\bm{\phi}(i)$
\WHILE{not converged}
    \STATE Sample batch and compute $\mathcal{L}_{\text{CFM}}$ (flow matching)
    \STATE Compute Lie brackets $[\bm{v}_i, \bm{v}_j]$ via autodiff
    \STATE Form residual Gram matrix $\bm{G}$ and extract top-$m$ eigenvectors $\bm{E}$
    \STATE Compute $\mathcal{L}_{\text{DAG}}$ using structure constants
    \STATE Update $\Theta$ via Natural Gradient descent (Fisher metric)
\ENDWHILE
    \STATE Extract graph from the learned geometry: edge $(i \to j)$ exists if $\mathbb{E}[\|\partial \bm{v}_j / \partial x_i\|] > \epsilon$ and $i < j$ in the learned topological order.
\STATE Return selected adjacency $\bm{A}$ and residual-footprint projector $\bm{E}$
\end{algorithmic}
\end{algorithm}

\paragraph{Running time.}
Let $d$ be the number of observed variables, $n$ the number of sample points at which fields are evaluated, $T$ the number of flow-matching epochs, $B$ the CFM batch size, $B_{\mathrm{Lie}}$ the mini-batch used for the Lie penalty, $q$ the interval between Lie-penalty updates, and $P_\Theta$ the number of network parameters. The outer Lie/Frobenius layer contains $\binom d2=O(d^2)$ unordered field-pair comparisons, but the dominant cost depends on whether fields are supplied by the kernel teacher or by the amortized network. The kernel teacher requires $d(d-1)$ local regressions in the fully pairwise implementation; with $a$ adjustment coordinates this is
\[
  O\!\left(d^2 n (n a + a^3)\right),
\]
or $O(dn(na+a^3))$ when a single anchor field is fitted per node and reused across targets. Given precomputed kernel fields, a NumPy Lie-bracket screen evaluates two smoothed directional derivatives per pair and all $d$ output coordinates, giving
\[
  O(d^3 n^2) + O(d^2 n s^3)
\]
time, where $s=2$ for the pair span and $s=d$ for the all-field visible span used in the Frobenius projection. The amortized \textsc{SKFM} network changes this scaling. CFM training costs $O(T B P_\Theta)$, while periodic autodiff Lie diagnostics contribute approximately
\[
  O\!\left((T/q)\,d^2 P_\Theta\right)
\]
for Jacobian-vector-product evaluations, plus $O((T/q)B_{\mathrm{Lie}}d^3)$ for bracket assembly and curvature accumulation under the diagonal or pairwise projection used in the scalable implementation. An all-span batched Frobenius solve can raise the algebraic projection term toward $O((T/q)B_{\mathrm{Lie}}d^5)$, which is why the released large-graph diagnostics use cheaper projections unless the graph is small. Spectral factorization of the $d\times d$ curvature Gram adds $O(d^3)$ after bracket energies are formed. Thus the direct amortized path has the schematic cost
\[
  \widetilde O\!\left(T B P_\Theta + (T/q)d^2 P_\Theta + d^3\right),
\]
with no parent-set enumeration term. If one uses \textsc{SKFM}/\textsc{Bridge} as a hybrid screen and then runs local BIC, the additional discrete term is $O(d r^k)$, or $O(d\,2^k)$ when the mask itself bounds each node to $k$ candidate parents.

This is the computational phase transition emphasized by the experiments. Kernel teachers can be quadratic in $n$ because every local regression compares evaluation points against samples; amortized fields are linear in the number of evaluated points once trained. Direct \textsc{SKFM} removes the bounded exponential search term entirely, but only yields a reliable DAG when the solvable-Lie triangular extraction is well posed. In less rigid regimes, the safer interpretation is categorical and algorithmic: polynomial Lie geometry supplies a high-recall screen, and the residual bounded exponential is the price of choosing a discrete DAG basis from continuous intervention geometry.

\begin{table}[t]
\centering
\scriptsize
\setlength{\tabcolsep}{3pt}
\renewcommand{\arraystretch}{1.25}
\begin{tabular}{>{\raggedright\arraybackslash}p{0.13\linewidth}
                >{\raggedright\arraybackslash}p{0.19\linewidth}
                >{\raggedright\arraybackslash}p{0.19\linewidth}
                >{\raggedright\arraybackslash}p{0.18\linewidth}
                >{\raggedright\arraybackslash}p{0.23\linewidth}}
\toprule
\textbf{Method} & \textbf{Training/precomputation} & \textbf{Per-iteration or screen} & \textbf{Typical total cost} & \textbf{Dominant tradeoff} \\
\midrule
NOTEARS
& None beyond data moments
& Score gradients $O(nd^2)$; acyclicity $O(d^3)$ for $h(W)=\operatorname{tr}(\exp(W\circ W))-d$
& $O(I(nd^2+d^3))$
& Smooth exact acyclicity avoids discrete enumeration, but the trace-exponential constraint creates a stiff nonconvex optimization problem. \\
\addlinespace
DirectLiNGAM
& Covariances/cumulants $O(nd^2)$
& Independence tests over $O(d)$ candidates across $d$ removal rounds
& $O(d^3 n)$ for linear tests; higher with kernel tests
& Efficient under linear non-Gaussian assumptions, but iterative variable removal and independence testing encode stronger functional assumptions. \\
\addlinespace
ICA-LiNGAM
& Whitening $O(nd^2)$
& FastICA-style component updates
& About $O(kd^2n)$ per component/update regime
& Non-Gaussianity orients edges, but full decomposition is cubic or worse in $d$ depending on ICA implementation and convergence. \\
\addlinespace
\textsc{Bridge} kernel screen
& Local regressions $O(d^2n(na+a^3))$
& Kernel brackets $O(d^3n^2)+O(d^2ns^3)$; local BIC over retained parents
& $O(d^2n(na+a^3))+O(d r^k)$ after screening
& Nonparametric field estimation is flexible but can be quadratic in $n$ because evaluation points are compared against samples. \\
\addlinespace
\textsc{Bridge} hybrid
& Amortized field learning $O(TBP_\Theta)$
& Autodiff brackets $O(d^3n)$; local BIC $O(d r^k)$, or $O(d\,2^k)$ if the mask bounds each node to $k$ parents
& $\widetilde O(TBP_\Theta+d^3n+d r^k)$
& The geometric screen converts global DAG search into bounded local parent-set search; the residual exponential is in $k\ll d$. \\
\addlinespace
\textsc{SKFM} end-to-end
& CFM plus Lie penalty $O(TBP_\Theta)+O((T/q)d^2P_\Theta)$
& Spectral Gram factorization $O(d^3)$; direct extraction $O(d^2)$
& $\widetilde O(TBP_\Theta+(T/q)d^2P_\Theta+d^3)$
& Fully polynomial after amortization, but direct triangular extraction is reliable only when the solvable-Lie basis is identifiable. \\
\bottomrule
\end{tabular}
\caption{Comparative computational tradeoffs for representative causal discovery families, including NOTEARS \citep{notears}, DirectLiNGAM \citep{shimizu2011directlingam}, and ICA-LiNGAM \citep{shimizu2006lingam}. Here $d$ is the number of observed variables, $n$ the sample size, $I$ the number of optimizer iterations, $T$ the number of flow-matching epochs, $B$ the batch size, $q$ the Lie-penalty update interval, $P_\Theta$ the network size, $a$ the local adjustment dimension, $s$ the Frobenius projection span, $r$ the maximum number of retained candidate parents per target, and $k$ the maximum downstream indegree. The practical distinction is that NOTEARS spends computation on a global smooth acyclicity constraint, LiNGAM spends it on linear non-Gaussian ordering assumptions, kernel \textsc{Bridge} spends it on nonparametric teacher fields, and hybrid \textsc{Bridge}/\textsc{SKFM} amortize the geometry so that the remaining combinatorics are bounded by the candidate mask.}
\label{tab:complexity_tradeoffs}
\end{table}

\begin{figure}
    \centering
    \includegraphics[width=0.85\linewidth]{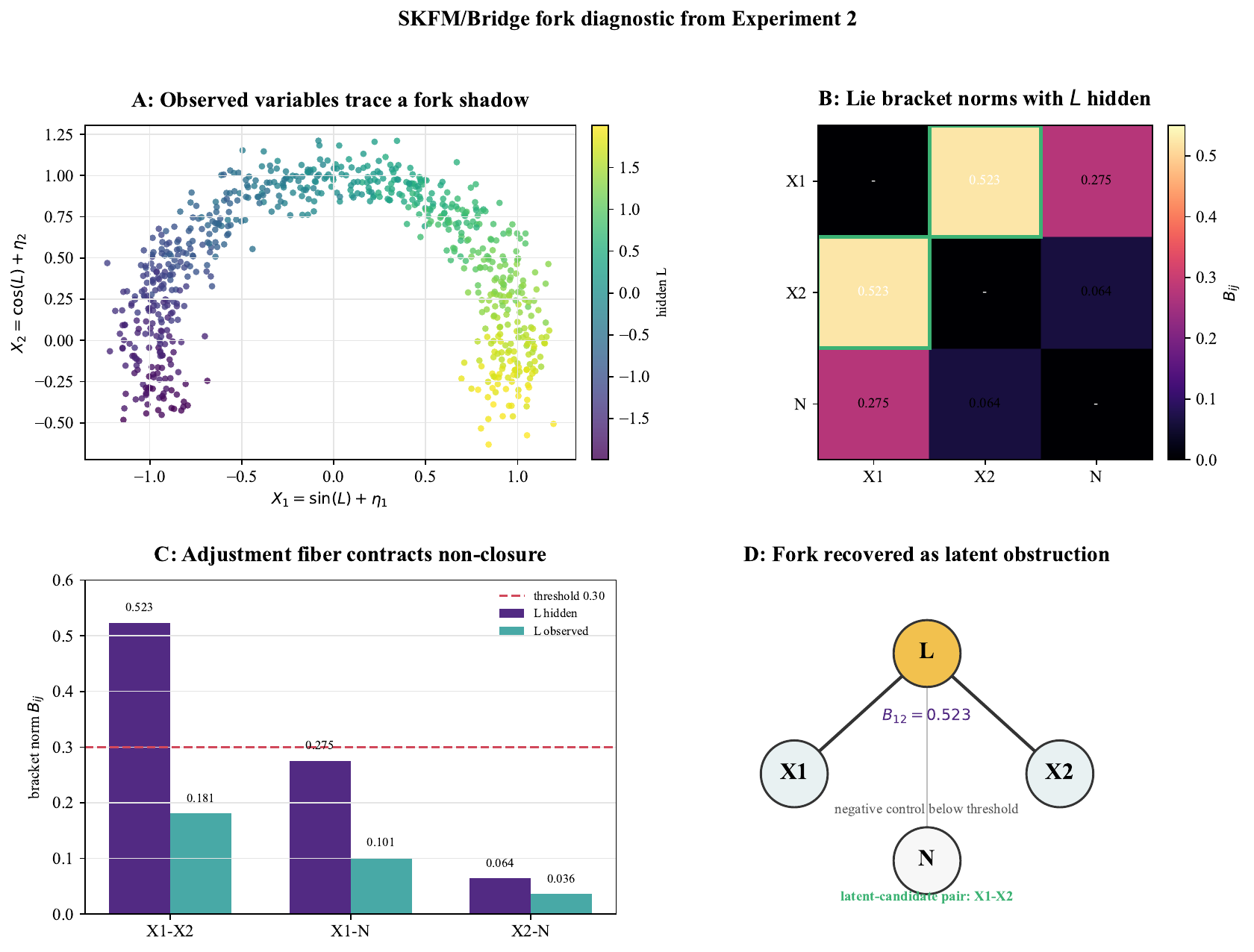}
    \caption{\textsc{SKFM}/\textsc{Bridge} fork diagnostic instantiated with the Experiment~2 nonlinear latent-fork data from Section~\ref{sec:experiments}. Panel A shows the observed variables $X_1=\sin(L)+\eta_1$ and $X_2=\cos(L)+\eta_2$, colored by the hidden driver $L$ for visualization. Panel B reports the hidden-latent Lie-bracket norms; the dominant non-closing pair is $(X_1,X_2)$ with $B_{12}=0.523$. Panel C compares hidden-$L$ and observed-$L$ regimes, showing that the same bracket contracts to $0.181$ when the missing adjustment fiber is supplied, while negative-control pairs remain below the latent-candidate threshold. Panel D summarizes the recovered fork interpretation $L\to X_1$ and $L\to X_2$ with no direct visible edge asserted between $X_1$ and $X_2$.}
    \label{fig:skfm}
\end{figure}

\subsection{Amortized Conditional Flow Matching}

\textsc{SKFM} replaces the Continuous Normalizing Flow (CNF) with \textbf{Conditional Flow Matching} (CFM) \citep{lipman2023flow}, which avoids ODE integration during training. The amortized vector field becomes:
\begin{equation}
    \frac{d}{dt}\bm{z}(t) = \bm{f}_\theta(\bm{z}(t), t, \bm{\phi}(i), \bm{m}_i, \bm{\epsilon}),
\end{equation}
where $\bm{\epsilon} \sim \mathcal{N}(0, \bm{I})$ is a noise source enabling exact simulation-free training. The conditional flow matching loss is:
\begin{equation}
    \mathcal{L}_{\text{CFM}} = \mathbb{E}_{t, i, \bm{x}} \left\| \bm{f}
    _\Theta(\bm{x}, t, \bm{\phi}(i), \cdot) - \bm{u}_t(\bm{x}|\bm{x}_1) \right\|^2_2,
\end{equation}
where $\bm{u}_t$ is the conditional vector field pushing base measure $\bm{p}_0$ to the interventional target $\bm{p}_{\text{do}(X_i)}$. This supplies an amortized transport-field interface. It does not automatically provide the same pointwise Radon--Nikodym ratio as a separately calibrated density model.

\subsection{Spectral Residual-Footprint Factorization}

We summarize learned nonclosure by a symmetric positive-semidefinite Gram operator built from vectorized projected residual features $b_i(x)$:
\begin{equation}
    \bm{G}_{ij}
    =
    \mathbb{E}_{\bm{x}\sim\bm{p}_{\text{obs}}}
    \left\langle b_i(\bm{x}),b_j(\bm{x})\right\rangle .
\end{equation}

Perform an eigendecomposition $\bm{G}=\bm{V}\bm{\Lambda}\bm{V}^\top$. Eigenvalue thresholding estimates the dimension of the \textbf{visible residual footprint}. It equals a number of latent loading directions only under the full-rank, positive-covariance, noncancellation, and eigengap assumptions below.

The projection
\begin{equation}
    \bm{z} = \bm{E}^\top \bm{x}, \quad \bm{E} = [\bm{e}_1 | \cdots | \bm{e}_m] \in \mathbb{R}^{d \times m},
\end{equation}
where $\bm e_k$ are the leading eigenvectors, is therefore a learned low-dimensional feature map. We do not call $\bm z$ an identified latent variable, assert a generative factorization through it, or use it as an automatically valid back-door adjustment set.

\subsection{Solvable Lie Algebra DAG Constraint}

Instead of post-hoc GES search, we enforce acyclicity via the \textbf{solvable Lie algebra condition}. A DAG corresponds to a nilpotent (strictly lower-triangular) adjacency structure in some basis. We impose this through a soft penalty on the structure constants of the Lie algebra.

Let $\bm{v}_i$ be the intervention fields. In the basis of the visible coordinates, compute the structure constants:
\begin{equation}
    [\bm{v}_i, \bm{v}_j] = \sum_{k=1}^d c^k_{ij} \bm{v}_k + \bm{r}_{ij}.
\end{equation}

For a DAG without latents, all $c^k_{ij} = 0$ when $i, j \leq k$ (respecting topological order). We define the \textbf{acyclicity loss}:
\begin{equation}
    \mathcal{L}_{\text{DAG}} = \sum_{i,j} \sum_{k \geq \max(i,j)} (c^k_{ij})^2 + \lambda_{\text{res}} \|\bm{r}\|_F^2.
\end{equation}
The first term encourages triangularity relative to the chosen order; the second penalizes residual nonclosure. Because nonclosure is nonspecific, the second term is a regularizer rather than a latent-confounding likelihood.

\subsection{Riemannian Optimization}

By Čencov's Theorem (Theorem~\ref{cencov}), the Fisher Information Metric $\bm{g}_F$ is the unique invariant metric. We optimize the parameters $\Theta$ using \textbf{Natural Gradient Descent}:
\begin{equation}
    \Theta_{t+1} = \Theta_t - \eta \, \bm{g}_F^{-1}(\Theta_t) \nabla_\theta \mathcal{L}_{\text{total}},
\end{equation}
where the total loss combines flow matching, calibration, and the DAG constraint:
\begin{equation}
    \mathcal{L}_{\text{total}} = \mathcal{L}_{\text{CFM}} + \alpha \mathcal{L}_{\text{cal}} + \beta \mathcal{L}_{\text{DAG}} + \gamma \mathcal{L}_{\text{Wasserstein}}.
\end{equation}
Here $\mathcal{L}_{\text{Wasserstein}}$ is the optimal transport regularizer for stability under weak overlap.

\subsection{Asymptotic Guarantees and Scope}

The propositions below isolate the idealized conditions behind the spectral and ordered-extraction components of \textsc{SKFM}. They are not finite-sample minimax or causal-identification results. Rank recovery assumes a low-rank factorization of the residual features and an eigenvalue gap; acyclicity assumes an order with respect to which extraction is triangular.

\begin{proposition}[Latent-Footprint Rank Consistency]
\label{latent_prop}
Let the population residual feature admit the factorization $b(X)=B\xi(X)$, where $B\in\mathbb R^{d\times m}$ has full column rank, $\mathbb E\|\xi(X)\|^2<\infty$, and $\Sigma_\xi=\mathbb E[\xi(X)\xi(X)^\top]$ is positive definite. If the empirical Gram matrix is the sample second moment of $b(X)$ and the threshold lies in the population eigengap after eigenvalue $m$, then spectral thresholding recovers $\operatorname{rank}(G)=m$ almost surely as $n\to\infty$. If the factors $\xi$ arise from noncancelling, linearly independent visible loadings of $m$ additive latent factors, this rank is their visible footprint dimension. The proposition does not infer that factorization from a nonzero bracket.
\end{proposition}

\begin{proof}
The population Gram matrix is
\[
G=\mathbb E[b(X)b(X)^\top]=B\Sigma_\xi B^\top .
\]
Full column rank of $B$ and positive definiteness of $\Sigma_\xi$ imply
\[
    \operatorname{rank}(G)=\operatorname{rank}(B\Sigma_\xi B^\top)=m.
\]
The strong law gives $\widehat G_n\to G$ almost surely entrywise and hence in operator norm for fixed $d$. Weyl's inequality then gives convergence of every empirical eigenvalue. A threshold lying strictly between $\lambda_m(G)$ and $\lambda_{m+1}(G)=0$ therefore selects exactly $m$ eigenvalues eventually almost surely.
\end{proof}

\begin{proposition}[Ordered Extraction Acyclicity]
\label{gradient_flow}
Fix an ordering of the variables. If Step~9 permits an edge only from an earlier to a later variable in that order, the extracted adjacency is acyclic. Every DAG can be represented by such an adjacency under one of its topological orders. A zero triangular penalty can encourage this condition but does not identify the correct order.
\end{proposition}

\begin{proof}
After permuting variables into the fixed order, every permitted edge lies strictly above the diagonal. A directed cycle would require at least one edge returning from a later vertex to an earlier one, which is forbidden. Conversely, a topological order of any DAG puts every edge from earlier to later. The result is graph theoretic and conditional on the order; it supplies no order-recovery theorem.
\end{proof}

\subsection{Capabilities and Current Boundary of \textsc{SKFM}}

\begin{enumerate}
    \item \textbf{Amortized fields:} CFM shares statistical strength across known intervention targets and avoids a separate kernel regression at each query.
    \item \textbf{Residual summary:} Spectral factorization produces a low-dimensional footprint whose latent interpretation is governed by Proposition~\ref{latent_prop}.
    \item \textbf{Ordered extraction:} Given a correct order, the triangular penalty supplies the acyclicity certificate in Proposition~\ref{gradient_flow}.
    \item \textbf{Open bottleneck:} Learned order and order-free graph extraction remain unstable on non-chain and random-DAG settings; Wasserstein regularization mitigates but does not solve weak overlap.
\end{enumerate}

\subsection{Summary}

By integrating bracket geometry into flow matching, \textsc{SKFM} provides a differentiable experiment in direct extraction. Its strongest established use remains field amortization and residual summarization. When a reliable order is unavailable, the paper's recommended pipeline is still \textsc{Bridge}: retain a high-recall mask and delegate identification and graph selection to a regime-appropriate downstream method.

\section{Illustrative Numerical Benchmarks and Simulation Experiments}
\label{sec:experiments}
To validate the structural components of \textsc{KDC}, this section provides concrete numerical simulations and benchmark experiments modeled after the core routines in our software framework. Unless explicitly labeled \textsc{SKFM}, the reported screening results are \textsc{Bridge} runs: an interventional density or response-field engine followed by Lie/Frobenius masking and, where needed, a downstream scorer. We present a multi-regime linear-Gaussian chain to evaluate structural edge scoring, a non-linear confounded fork to highlight Lie bracket residual tracking under latent conditions, a \textsc{Bridge}-pruned TCES search that tests whether the framework can reduce the effective DAG search space in practice, an FCI skeleton baseline on the same latent-chain sweep, a non-chain latent-DAG motif stress test connecting the experiments to the spectral ideas in Section~\ref{sec:scaling}, an \textsc{SKFM} non-chain graph-extraction ablation, a Sachs protein-signaling benchmark comparing the \textsc{Bridge} candidate family against GES, a real S9 gene-expression pilot without a gold DAG, and a \textsc{Bridge}-DCDI integration study that connects the same geometric screen to differentiable causal discovery.

\subsection{Experiment 1: Multi-Regime Causal Chain Directionality}

We construct a synthetic three-node causal chain $X_0 \to X_1 \to X_2$ governed by structural equations across distinct regimes (observational and soft-interventional). The baseline data-generating process is defined as:
\begin{align}
    X_0 &\sim \mathcal{N}(0, 1.0) \nonumber \\
    X_1 &= 0.8 X_0 + \epsilon_1, \quad \epsilon_1 \sim \mathcal{N}(0, 0.5) \\
    X_2 &= 0.6 X_1 + \epsilon_2, \quad \epsilon_2 \sim \mathcal{N}(0, 0.5) \nonumber
\end{align}
A localized soft intervention is simulated on $X_1$ by shifting its local mechanism to $X_1 \sim \mathcal{N}(1.5, 0.2)$, inducing a change of measure across the downstream manifold.

For each soft intervention we fit a regularized Gaussian density model for the observational and interventional regimes and evaluate the Radon--Nikodym weights
$\rho_i(z)=p_{\mathrm{do}(X_i)}(z)/p_{\mathrm{obs}}(z)$ on observational samples. The reported score is the RN-weighted standardized response
\begin{equation}
    S_{\mathrm{RN}}(X_i \to X_j)
    =
    \frac{\left|\mathbb{E}_{\mathrm{obs}}\left[\rho_i(Z)X_j\right]
    - \mathbb{E}_{\mathrm{obs}}[X_j]\right|}
    {\sqrt{\mathrm{Var}_{\mathrm{obs}}(X_j)}} .
\end{equation}
This score is not used as a full graph estimator; it is a directionality diagnostic measuring whether an intervention on $X_i$ transports mass in the coordinate $X_j$.

\begin{table}[h]
\centering
\caption{\textsc{Bridge} RN-weighted response scores $S_{\mathrm{RN}}(X_i \to X_j)$ on a three-node chain.}
\label{tab:exp1_results}
\vspace{0.5em}
\begin{tabular}{lccc}
\toprule
\textbf{Source Node} & \textbf{Target $X_0$} & \textbf{Target $X_1$} & \textbf{Target $X_2$} \\
\midrule
$X_0$ (Intervened) & — & \textbf{1.113} & 0.759 \\
$X_1$ (Intervened) & 0.164 & — & \textbf{0.974} \\
$X_2$ (Intervened) & 0.133 & 0.241 & — \\
\bottomrule
\end{tabular}
\end{table}

As illustrated in Table~\ref{tab:exp1_results}, reweighting the sample space by the parent's causal density produces a substantially larger standardized response in its descendants than in its ancestors. The asymmetry between $S_{\mathrm{RN}}(X_0 \to X_1)=1.113$ and $S_{\mathrm{RN}}(X_1 \to X_0)=0.164$ recovers the first arrow, while $S_{\mathrm{RN}}(X_1 \to X_2)=0.974$ exceeds $S_{\mathrm{RN}}(X_2 \to X_1)=0.241$ for the second. Selecting the best incoming source above a threshold of $0.25$ gives the exact visible chain $X_0 \to X_1 \to X_2$ with $\mathrm{SHD}=0$ and $F_1=1.0$. The non-adjacent ancestral score $S_{\mathrm{RN}}(X_0 \to X_2)=0.759$ remains positive, reflecting indirect downstream propagation through the chain.

\subsection{Experiment 2: Sensitivity of the Closure Residual in a Latent-Fork Simulation}
To test whether the closure residual responds to an omitted common driver in a controlled example, we simulate a nonlinear fork with an unobserved latent factor $L$:
\begin{align}
    L &\sim \text{Uniform}(-2, 2) \nonumber \\
    X_1 &= \sin(L) + \eta_1, \quad \eta_1 \sim \mathcal{N}(0, 0.1) \\
    X_2 &= \cos(L) + \eta_2, \quad \eta_2 \sim \mathcal{N}(0, 0.1) \nonumber
\end{align}
Because $L$ is omitted from the visible dataset, passive conditioning suggests a strong nonlinear correlation between $X_1$ and $X_2$. We include an independent negative-control coordinate $N \sim \mathcal{N}(0,1)$ and estimate local visible intervention vector fields using smoothed local linear derivatives. We then compute the empirical bracket norm
\begin{equation}
    B_{ij} = \mathbb{E}_{z}\left[\left\|[v_i,v_j](z)\right\|_2\right],
\end{equation}
and compare the hidden-latent regime against a diagnostic regime in which $L$ is supplied as an observed adjustment fiber. If the bracket signal is genuinely induced by the latent common cause, exposing $L$ should reduce the non-commutativity between $X_1$ and $X_2$.

\begin{table}[h]
\centering
\caption{\textsc{Bridge} Lie-bracket diagnostics for a nonlinear latent fork.}
\label{tab:exp2_lie_bracket}
\vspace{0.5em}
\begin{tabular}{lccc}
\toprule
\textbf{Visible Pair} & \textbf{Hidden $L$: $B_{ij}$} & \textbf{Observed $L$: $B_{ij}$} & \textbf{Hidden residual ratio} \\
\midrule
$(X_1,X_2)$ & \textbf{0.523} & 0.181 & 0.240 \\
$(X_1,N)$ & 0.275 & 0.101 & 0.516 \\
$(X_2,N)$ & 0.064 & 0.036 & 0.721 \\
\bottomrule
\end{tabular}
\end{table}

As shown in Table~\ref{tab:exp2_lie_bracket}, the dominant non-commutativity occurs on the confounded visible pair $(X_1,X_2)$, with $B_{12}=0.523$. When the latent driver $L$ is revealed as an observed adjustment coordinate, this bracket norm drops to $0.181$. The negative-control pairs remain substantially weaker, and using a bracket threshold of $0.30$ flags only $(X_1,X_2)$ as a latent-confounding candidate. This experiment turns the qualitative Lie-bracket claim into a concrete diagnostic: hidden common causes induce measurable non-commutativity in the visible intervention fields, and the signal contracts when the missing adjustment fiber is restored.

\subsection{Experiment 3: \textsc{Bridge}-Pruned TCES Search Avoids DAG Enumeration}

The central computational bottleneck in score-based causal discovery is the size of the model class. Even for a modest number of observed variables, the number of labeled DAGs grows super-exponentially. Score-based methods such as GES reduce this burden by searching over Markov equivalence classes (MECs), using covered edge reversals and CPDAG/essential-graph representatives rather than treating every labeled DAG as statistically distinct \citep{chickering:jmlr}. This quotient is important but it is not an exponential collapse: Gillispie and Perlman's enumeration found that the ratio of DAGs to Markov equivalence classes appears to approach about $3.7$, and Schmid and Sly prove that the ratio of Markov equivalence classes to DAGs converges to a positive constant asymptotically \citep{gillispie2002size,schmid2024number}. Thus GES changes the constant in the search-space scale, whereas the \textsc{Bridge} mask changes the admissible arrow family before score evaluation begins.

To test whether the Lie/Frobenius construction has practical algorithmic force, we implemented a two-stage discovery pipeline:
\begin{enumerate}
    \item \textbf{Information-geometric screening:} estimate local visible intervention vector fields, compute pairwise Frobenius anisotropy, and retain only high-influence directed arrows in a candidate mask.
    \item \textbf{TCES/BIC selection:} score only the parent sets allowed by this mask using a Gaussian BIC objective, with optional Topos-Constraint-Enhanced Structure (TCES) sheaf and $j$-stability regularizers.
\end{enumerate}
Thus the geometric layer is not treated as the final estimator. Its role is to convert the original DAG search problem into a much smaller constrained score-based problem.

We generated $n=220$ samples from a seven-node nonlinear chain with one hidden confounding source:
\begin{align}
    L &\sim \mathcal{N}(0,1), \nonumber \\
    X_0 &= 0.7L + \epsilon_0, \nonumber \\
    X_1 &= 0.8\tanh(X_0) + \epsilon_1, \nonumber \\
    X_2 &= 0.9\sin(X_1) + \epsilon_2, \nonumber \\
    X_3 &= 0.7X_2 + \epsilon_3, \nonumber \\
    X_j &= 0.75\tanh(X_{j-1}) + \epsilon_j,\qquad j=4,5, \nonumber \\
    X_6 &= 0.75\tanh(X_5) + 0.7\sin(L) + \epsilon_6,
\end{align}
with independent Gaussian noise terms. The visible ground-truth graph is the chain
$X_0 \to X_1 \to X_2 \to X_3 \to X_4 \to X_5 \to X_6$, while $L$ remains unobserved and creates a latent confounding distortion between the first and last visible variables, $X_0$ and $X_6$.

\begin{figure}[t]
\centering
\includegraphics[width=0.88\textwidth]{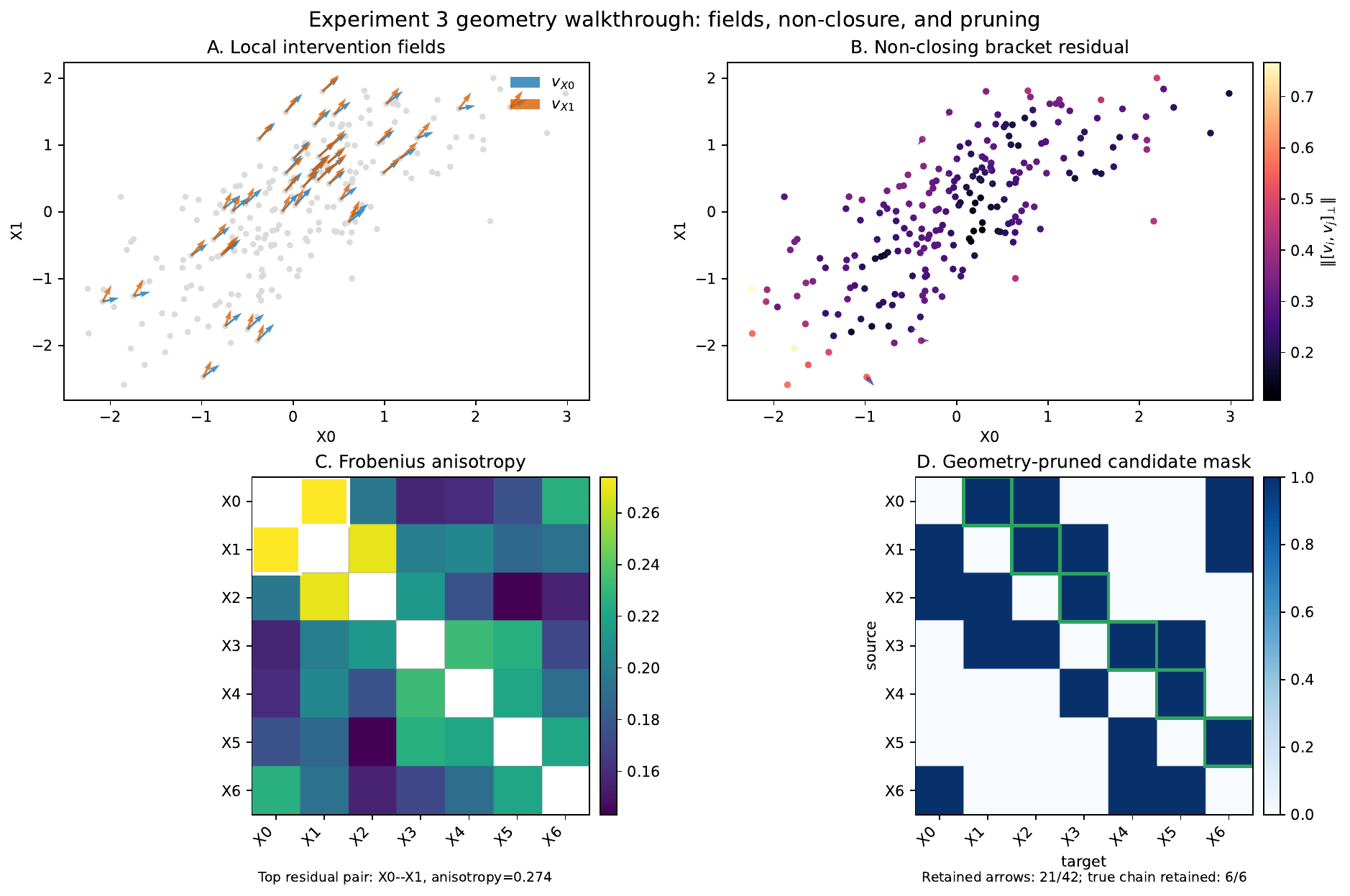}
\caption{\textsc{Bridge} geometry walkthrough for the seven-node latent-confounded chain in Experiment~3. The panels show learned local intervention fields, the non-closing Lie-bracket residual $[v_i,v_j]_\perp$, pairwise Frobenius anisotropy, and the final \textsc{Bridge}-pruned candidate mask. Green outlines mark the true visible chain arrows, all retained before TCES/BIC scoring.}
\label{fig:exp3_geometry_walkthrough}
\end{figure}

The geometric screen retained $21$ directed candidate arrows out of the $42$ possible directed non-self pairs while preserving all six true visible arrows. This high-recall candidate family was then passed to a TCES/BIC scorer with maximum indegree one. The resulting search evaluates only the mask-constrained parent-set combinations rather than the full labeled DAG family or its constant-factor Markov-equivalence quotient.

\begin{table}[h]
\centering
\caption{\textsc{Bridge}-pruned TCES search on a seven-node latent-confounded chain.}
\label{tab:geometry_tces_pruning}
\vspace{0.5em}
\begin{tabular}{lr}
\toprule
\textbf{Quantity} & \textbf{Value} \\
\midrule
Possible directed non-self arrows & $42$ \\
\textsc{Bridge}-retained candidate arrows & $21$ \\
True directed arrows retained by screen & $6/6$ \\
All labeled DAGs on seven nodes & $1{,}138{,}779{,}265$ \\
Full-space MEC scale (DAGs$/3.7$ heuristic) & $\approx 3.08{\times}10^8$ \\
Acyclic DAGs inside geometry mask & $107{,}121$ \\
Exact acyclic-family compression & $10{,}630.8\times$ \\
TCES parent-set combinations evaluated & $16{,}384$ \\
TCES acyclic graphs evaluated & $6{,}650$ \\
\bottomrule
\end{tabular}
\end{table}

The final selected graph preserved most of the visible chain but no longer exactly matched it:
\[
    \mathrm{SHD}=2,\qquad
    F_1 \approx 0.833,\qquad
    \mathrm{TP}=5,\quad \mathrm{FP}=1,\quad \mathrm{FN}=1.
\]
This experiment provides concrete evidence that the information-geometric construction can do more than diagnose latent curvature: it can make causal discovery more tractable by converting an intractable global DAG search into a high-recall, \textsc{Bridge}-pruned scoring problem. At the same time, the endpoint-confounded chain is a harder case than the earlier interior-confounded chain: Lie/Frobenius geometry supplies the search-space reduction and retains the truth, but TCES/BIC still has to resolve a long-range latent distortion inside the reduced family.

To check that this behavior was not an isolated seven-node artifact, we repeated the same endpoint-confounded latent-chain construction while increasing the number of observed chain nodes and holding the screen hyperparameters fixed. In the size sweep, the hidden source acts on $X_0$ and $X_{d-1}$, leaving the visible ground-truth chain unchanged. Table~\ref{tab:geometry_tces_size_sweep} reports the resulting size sweep. In every case, the Lie/Frobenius screen retained all true visible chain arrows, while the exact max-indegree-one TCES/BIC search made two structural errors after scoring inside the reduced family. For $d=9$ and $d=10$, exact enumeration of every acyclic subgraph in the geometric mask was intentionally not performed because the mask exceeded the counting cap; the table therefore reports the conservative compression lower bound against all mask subgraphs together with the exact number of TCES parent-set combinations evaluated.

\begin{table}[h]
\centering
\caption{Size sweep for \textsc{Bridge}-pruned TCES search on latent-confounded chains.}
\label{tab:geometry_tces_size_sweep}
\vspace{0.5em}
\resizebox{\textwidth}{!}{%
\begin{tabular}{rrrrrrrr}
\toprule
$d$ & Directed arrows & \textsc{Bridge} arrows & True retained & Compression lower bound & TCES evals & SHD & $F_1$ \\
\midrule
$7$  & $42$ & $21$ & $6/6$ & $5.43{\times}10^2$ & $16{,}384$ & $2$ & $0.833$ \\
$8$  & $56$ & $24$ & $7/7$ & $4.67{\times}10^4$ & $65{,}536$ & $2$ & $0.857$ \\
$9$  & $72$ & $27$ & $8/8$ & $9.04{\times}10^6$ & $262{,}144$ & $2$ & $0.875$ \\
$10$ & $90$ & $30$ & $9/9$ & $3.89{\times}10^9$ & $1{,}048{,}576$ & $2$ & $0.889$ \\
\bottomrule
\end{tabular}
}
\end{table}

Because the data-generating process contains an unobserved common cause, unconstrained GES is not the most appropriate standalone baseline: it searches DAGs or equivalence classes under causal sufficiency, whereas latent-aware constraint and hybrid methods such as FCI, RFCI, and GFCI return partial ancestral graph structure with uncertain, directed, and bidirected endpoint marks \citep{spirtes2000causation,colombo2012learning,ogarrio2016hybrid}. We therefore add a lightweight FCI baseline on the same latent-chain sweep, implemented with \texttt{causal-learn} \citep{zheng2024causallearn}. Since FCI returns a PAG rather than a single directed visible DAG, Table~\ref{tab:latent_chain_fci_baseline} reports skeleton recovery. This is intentionally a different target than the \textsc{Bridge}-TCES result above: FCI asks whether the latent-aware PAG skeleton contains the right adjacencies, while \textsc{Bridge} asks whether local Lie/Frobenius geometry can preserve the true directed support and shrink the family passed to a downstream scorer.

\begin{table}[h]
\centering
\caption{Latent-aware FCI baseline on the same latent-confounded chain sweep. FCI returns a PAG, so metrics are skeleton-level and should not be read as directed-DAG recovery.}
\label{tab:latent_chain_fci_baseline}
\vspace{0.5em}
\begin{tabular}{rrrrrrrr}
\toprule
$d$ & True skel. & FCI skel. & TP & FP & FN & Skel. $F_1$ & Skel. SHD \\
\midrule
$7$  & $6$ & $7$  & $6$ & $1$ & $0$ & $0.923$ & $1$ \\
$8$  & $7$ & $8$  & $7$ & $1$ & $0$ & $0.933$ & $1$ \\
$9$  & $8$ & $9$  & $8$ & $1$ & $0$ & $0.941$ & $1$ \\
$10$ & $9$ & $10$ & $9$ & $1$ & $0$ & $0.947$ & $1$ \\
\bottomrule
\end{tabular}
\end{table}

The comparison is useful precisely because FCI performs well here: it retains every true visible skeleton edge and adds only one extra adjacency at each size. Thus the claim of Table~\ref{tab:geometry_tces_size_sweep} should not be that GES/TCES are superior latent-confounded discovery algorithms in isolation. Rather, TCES/GES is the final scoring layer inside a \textsc{Bridge}-defined candidate family. The latent-aware baseline supplies the appropriate PAG/skeleton comparison, while \textsc{Bridge} supplies a different computational object: a directed admissible-arrow mask that reduces the downstream DAG family by factors from $5.43{\times}10^2$ to $3.89{\times}10^9$ on this sweep while retaining all true visible arrows.

A sharper directed latent-confounded comparator is the higher-order-cumulant method of \citet{cai2023cumulants}. That method assumes a canonical latent-variable LiNGAM model, uses higher-order cumulants to obtain closed-form mixing-coefficient estimates in One-Latent-Component structures, and then recursively identifies latent components and causal orders. Its assumptions include at least one pure observed set for each latent confounder, at least three observed children per latent confounder, and faithfulness. Under those assumptions, the published Case~7 benchmark reports directed-edge $F_1$ scores of $0.55$, $0.61$, and $0.73$ at sample sizes $500$, $1000$, and $2000$, respectively. We do not insert those numbers into Table~\ref{tab:geometry_tces_size_sweep} because the present \textsc{Bridge} chain is a nonlinear Gaussian-noise benchmark and the hidden source directly affects only two visible sites, violating the cumulant method's three-child latent condition.

To make this comparison executable with standard software, we also ran baselines from the Python \texttt{lingam} package \citep{ikeuchi2023lingam}. Table~\ref{tab:lingam_linear_baselines} reports RCD \citep{maeda2020rcd} and CAM-UV \citep{maeda2021camuv} on a linear non-Gaussian chain with the same visible support and a hidden common cause between the endpoint variables $X_0$ and $X_{d-1}$. This benchmark is intentionally matched to the LiNGAM family rather than to the nonlinear Gaussian-noise chain in Table~\ref{tab:geometry_tces_size_sweep}. RCD recovers most directed chain edges for the larger instances and sometimes marks pairs as latent/unknown through \texttt{NaN} entries; CAM-UV is more conservative, with fewer false positives but lower recall. These baselines are not search-compression methods, so their role is complementary: they estimate a graph or partial graph directly, while \textsc{Bridge} reduces the admissible directed family before TCES/GES/DCDI-style scoring.

\begin{table}[h]
\centering
\caption{Latent-confounder baselines from the local \texttt{lingam} package on a linear non-Gaussian latent chain ($n=500$). Metrics are directed visible-edge recovery; \texttt{NaN} pairs are pairs marked unknown/confounded by the estimator.}
\label{tab:lingam_linear_baselines}
\vspace{0.5em}
\begin{tabular}{rlrrrrr}
\toprule
$d$ & Method & Directed edges & \texttt{NaN} pairs & Prec. & Recall & $F_1$ \\
\midrule
$7$  & RCD    & $6$  & $2$ & $0.833$ & $0.833$ & $0.833$ \\
$7$  & CAM-UV & $3$  & $0$ & $1.000$ & $0.500$ & $0.667$ \\
$8$  & RCD    & $7$  & $4$ & $0.714$ & $0.714$ & $0.714$ \\
$8$  & CAM-UV & $4$  & $0$ & $0.750$ & $0.429$ & $0.545$ \\
$9$  & RCD    & $7$  & $2$ & $1.000$ & $0.875$ & $0.933$ \\
$9$  & CAM-UV & $5$  & $0$ & $0.800$ & $0.500$ & $0.615$ \\
$10$ & RCD    & $9$  & $2$ & $0.889$ & $0.889$ & $0.889$ \\
$10$ & CAM-UV & $5$  & $0$ & $0.800$ & $0.444$ & $0.571$ \\
\bottomrule
\end{tabular}
\end{table}

We nevertheless keep the installed \texttt{CEId-from-Moments} codebase for \citet{tramontano2025ceid} as a second pilot. That repository is primarily an effect-identification package for proxy and instrumental-variable motifs, but it also ships the ReLVLiNGAM high-order-cumulant estimator. We therefore generated a linear non-Gaussian chain with one hidden common cause between the endpoint variables $X_0$ and $X_{d-1}$, ran the same \textsc{Bridge} geometric screen, and evaluated ReLVLiNGAM as a visible directed-edge estimator. Table~\ref{tab:relvlingam_bridge_pilot} should be read as a compatibility pilot rather than a definitive benchmark: the cumulant estimator is being used outside the exact proxy/IV target of \citet{tramontano2025ceid}, while \textsc{Bridge} is being tested on a linear model rather than the nonlinear intervention-field setting of Table~\ref{tab:geometry_tces_size_sweep}. The comparison is somewhat platform-sensitive because the upstream ReLVLiNGAM implementation depends on a compiled Cython moment-estimation extension; the numbers below were obtained with the local macOS build. The useful outcome is that the comparison is now executable where the CEId extension is available and exposes the right next step: a dedicated canonical-lvLiNGAM benchmark, preferably using the cumulant discovery assumptions of \citet{cai2023cumulants}, where \textsc{Bridge} reports mask recall and candidate-family compression and the cumulant method reports final directed-graph recovery.

\begin{table}[h]
\centering
\caption{Pilot linear non-Gaussian latent-chain comparison using the installed ReLVLiNGAM implementation from \texttt{CEId-from-Moments}. \textsc{Bridge} columns report candidate-mask size and search compression; ReLVLiNGAM columns report visible directed-edge recovery after thresholding estimated coefficients at $10^{-2}$. The run depends on the platform-specific compiled CEId moment-estimation extension.}
\label{tab:relvlingam_bridge_pilot}
\vspace{0.5em}
\begin{tabular}{rrrrrr}
\toprule
$d$ & \textsc{Bridge} cand. & True kept & Comp. & ReLV latent cols. & ReLV $F_1$ \\
\midrule
$7$  & $21$ & $3/6$ & $5.43{\times}10^2$ & $1$ & $0.222$ \\
$8$  & $24$ & $6/7$ & $4.67{\times}10^4$ & $1$ & $0.294$ \\
$9$  & $27$ & $5/8$ & $9.04{\times}10^6$ & $2$ & $0.186$ \\
$10$ & $30$ & $5/9$ & $3.89{\times}10^9$ & $1$ & $0.192$ \\
\bottomrule
\end{tabular}
\end{table}

The endpoint-confounded linear pilot is therefore more difficult than the interior-pair pilot: the same \textsc{Bridge} hyperparameters still compress aggressively, but mask recall varies from $0.50$ to $0.86$, and the cumulant estimator tends to return dense visible graphs with low precision. We include the row because it is executable and diagnostic, not because it is a settled benchmark.

As a first end-to-end comparison with the \textsc{SKFM} method described in Section~\ref{sec:skfm}, we also ran the same latent-chain scale-up using amortized conditional-flow fields, the spectral residual-footprint projector, the soft Lie-algebra DAG penalty, and a direct order-local graph extractor rather than the TCES/BIC scorer. The extractor used the same known chain order and max-indegree-one structural prior used by the TCES scale-up, selecting the nearest order-local parent when its learned influence was within $60\%$ of the best admissible parent. This is therefore not a replacement for the general \textsc{Bridge}-pruned scorer, but a controlled test of whether \textsc{SKFM} can close the loop on the same ordered-chain benchmark without enumerating candidate DAGs.

\begin{table}[h]
\centering
\caption{End-to-end \textsc{SKFM} on the Table~\ref{tab:geometry_tces_size_sweep} endpoint-confounded latent-chain scale-up. Runs used $n=220$, $200$ field-matching epochs, the spectral residual-rank threshold described in Section~\ref{sec:skfm}, and an order-local max-indegree-one extractor.}
\label{tab:skfm_table5_scaleup}
\vspace{0.5em}
\begin{tabular}{rrrrrrr}
\toprule
$d$ & Learned arrows & Field corr. & Rel. infl. MSE & Residual rank & SHD & $F_1$ \\
\midrule
$7$  & $6$ & $0.955$ & $0.018$ & $2$ & $0$ & $1.000$ \\
$8$  & $7$ & $0.935$ & $0.052$ & $1$ & $0$ & $1.000$ \\
$9$  & $8$ & $0.943$ & $0.049$ & $1$ & $0$ & $1.000$ \\
$10$ & $9$ & $0.930$ & $0.089$ & $1$ & $0$ & $1.000$ \\
\bottomrule
\end{tabular}
\end{table}

Table~\ref{tab:skfm_table5_scaleup} shows that, once the ordered-chain prior is supplied, \textsc{SKFM} recovers the exact visible chain for all four sizes while bypassing the downstream combinatorial scorer. The learned field quality remains high across the sweep, with influence-score correlations around $0.93$--$0.96$. The selected residual footprint is usually one-dimensional, with the seven-node endpoint-confounded run selecting rank two. These ranks summarize fitted nonclosure and are not estimates of the literal number of hidden variables. This result motivates the scoped \textsc{SKFM} direction while also clarifying its boundary: the order-local extractor is tailored to this chain benchmark.

\subsection{Experiment 4: Non-Chain Latent DAG Motifs and Spectral Curvature Rescue}
The chain experiments above establish that Lie/Frobenius screening can preserve directed support in a path-like latent-confounded system. A natural concern is that this success may be specific to chains, where the strongest downstream influence often coincides with the next structural arrow. To test whether the geometric signal extends beyond chains, we generated three six-node nonlinear DAG motifs with hidden common causes:
\begin{enumerate}
    \item \textbf{Diamond:} $X_0$ forks to $X_1$ and $X_2$, which rejoin at $X_3$.
    \item \textbf{Collider/fork:} two roots collide at $X_2$, then split and rejoin at $X_5$.
    \item \textbf{Multi-branch:} an irregular DAG with both skip-like branching and a terminal join.
\end{enumerate}
Each motif includes one or two injected latent-confounded visible pairs. We evaluate two versions of the candidate screen: the original top-$k$ local influence mask, and a spectral curvature-footprint rescue inspired by the Lie-bracket spectrum of Section~\ref{sec:scaling}. The spectral version forms a curvature matrix from Frobenius anisotropy, constructs the Gram proxy $G=A A^\top$, extracts the leading curvature eigenvectors, and treats their largest coordinates as latent footprints. Candidate arrows are then rescued only when supported by influence rankings as common-source or shared-child arrows around those low-rank footprints.

\begin{table}[h]
\centering
\caption{\textsc{Bridge} non-chain latent-DAG motif stress test. Values are means over three random seeds with $n=160$. Candidate recall measures the fraction of true visible arrows retained by the \textsc{Bridge} mask before downstream scoring.}
\label{tab:latent_dag_motifs}
\vspace{0.5em}
\resizebox{\textwidth}{!}{%
\begin{tabular}{lrrrrr}
\toprule
\textbf{Motif} & \textbf{Plain recall} & \textbf{Plain arrows} & \textbf{Spectral recall} & \textbf{Spectral arrows} & \textbf{Residual rank} \\
\midrule
Diamond & $0.722$ & $21.7$ & $\mathbf{1.000}$ & $28.7$ & $6.33$ \\
Collider/fork & $0.944$ & $21.7$ & $\mathbf{1.000}$ & $29.0$ & $2.67$ \\
Multi-branch & $0.714$ & $22.3$ & $\mathbf{0.952}$ & $28.3$ & $7.00$ \\
\bottomrule
\end{tabular}
}
\end{table}

Table~\ref{tab:latent_dag_motifs} shows that the plain influence mask is not yet robust for arbitrary non-chain DAGs. Diamond joins and irregular multi-branch structures expose a real failure mode: true edges near latent footprints can be ranked below stronger but less structural local influences. The spectral curvature rescue repairs much of this failure. It restores perfect candidate recall for the diamond and collider/fork motifs and raises multi-branch recall from $0.714$ to $0.952$. This supports the spectral-scaling proposal in Section~\ref{sec:scaling}: latent curvature should be treated as a low-rank footprint rather than merely as a pairwise anomaly.

The cost is that the rescued candidate family is wider: in these six-node motifs, the spectral mask keeps roughly $28$--$29$ of the $30$ possible directed non-self arrows. Thus this experiment is not a headline compression result. Its value is diagnostic: it identifies where the first-order influence screen fails, and it demonstrates that the Lie-bracket spectrum points in the right direction for recovering hidden-root, join, and collider neighborhoods. Downstream greedy BIC scoring inside these masks still fails to recover the exact nonlinear latent DAGs, even when all true arrows are retained. This separates the next research tasks cleanly: spectral curvature improves the candidate family, while nonlinear local scoring or TCES-style stability terms are still needed for final graph selection.

We also ran the \textsc{SKFM} system (Section~\ref{sec:skfm}) on the same three motifs as an end-to-end diagnostic. Table~\ref{tab:skfm_motif_extraction_ablation} separates the field-learning layer from the final graph-extraction heuristic. The learned amortized fields remain accurate, with mean influence-score correlations between $0.926$ and $0.960$ across the motifs. However, candidate recall in the learned Stage-1 screen is uneven, especially for the diamond and multi-branch motifs. The default top-$k$ extractor therefore recovers only partial directed structure. Supplying the true topological order and using a wider local multi-parent extractor improves recall on every motif, but it also admits extra arrows, so $F_1$ remains around $0.71$--$0.72$ rather than reaching exact DAG recovery.

\begin{table}[h]
\centering
\caption{\textsc{SKFM} non-chain DAG extraction ablation. Values are means over seeds $11$--$13$ with $n=160$. ``Default'' is the direct top-$k$ extractor. ``Local/order'' supplies the motif's coordinate topological order and uses a local two-parent extractor; it is an oracle-order diagnostic, not a downstream score-based search.}
\label{tab:skfm_motif_extraction_ablation}
\vspace{0.5em}
\resizebox{\textwidth}{!}{%
\begin{tabular}{lrrrrrr}
\toprule
\textbf{Motif} & \textbf{Field corr.} & \textbf{Stage-1 recall} & \textbf{Default $F_1$} & \textbf{Default recall} & \textbf{Local/order $F_1$} & \textbf{Local/order recall} \\
\midrule
Diamond & $0.954$ & $0.611$ & $0.614$ & $0.667$ & $\mathbf{0.714}$ & $\mathbf{0.833}$ \\
Collider/fork & $0.960$ & $0.833$ & $0.634$ & $0.722$ & $\mathbf{0.724}$ & $\mathbf{0.889}$ \\
Multi-branch & $0.926$ & $0.571$ & $0.700$ & $0.667$ & $\mathbf{0.707}$ & $\mathbf{0.762}$ \\
\bottomrule
\end{tabular}
}
\end{table}

The diagnostic above localized the remaining bottleneck to Step~9 of Algorithm~\ref{alg:skfm}: the raw Jacobian score
\[
    I_{ij}=\mathbb{E}_z\left[\left\|\frac{\partial v_j(z)}{\partial x_i}\right\|_F\right]
\]
measures total geometric influence, so it can be large for direct arrows, indirect ancestral paths, and confounded pairs. We therefore ran focused Step~9 calibrations on the diamond motif, which contains a common-cause/sibling obstruction, and the collider/fork motif, which contains the opposite common-effect geometry. Both runs used free-vector SKFM fields with a continuity regularizer
\[
    \mathcal{L}_{\mathrm{div}}
    =
    \sum_i
    \mathbb{E}_{z\sim p_{\mathrm{obs}}}
    \left[
    \left(
    \nabla\!\cdot v_i(z)-(\rho_i(z)-1)
    \right)^2
    \right],
\]
with weight $5$. This raised held-out continuity correlations above $0.7$ in the motif diagnostics, making the learned $v_i$ fields behave more like calibrated intervention transports rather than generic vector-valued regressors.

The calibrated Step~9 extractor uses the leading spectral curvature subspace $E$ with rank $m=2$ and defines a node-contamination score
\[
    c_i
    =
    \frac{\mathbb{E}_z\left[\|E^\top v_i(z)\|_2\right]}
         {\mathbb{E}_z\left[\|v_i(z)\|_2\right]+\epsilon}.
\]
Candidate arrows are rescored by a contamination-gradient boost and a latent-bracket penalty,
\[
    S_{ij}
    =
    I_{ij}
    \exp\!\left(\alpha\,\max(c_j-c_i,0)\right)
    \exp\!\left(-\lambda_{\mathrm{lat}} R^{\mathrm{lat}}_{ij}\right),
    \qquad
    R^{\mathrm{lat}}_{ij}
    =
    \frac{\mathbb{E}_z\left[\|E^\top [v_i,v_j](z)\|_2\right]}
         {\mathbb{E}_z\left[\|[v_i,v_j](z)\|_2\right]+\epsilon}.
\]
The boost, with $\alpha\simeq 4$, prefers parents that are cleaner than their effects. The penalty suppresses edges whose bracket energy lies in the spectral latent footprint. Finally, two post-processing steps convert the continuous score into an adjacency matrix. A root-pair mask removes an ordered edge $i\to j$ when $|c_i-c_j|<\epsilon_{\mathrm{root}}$ and $R^{\mathrm{lat}}_{ij}>\tau_{\mathrm{lat}}$, preventing the extractor from forcing a spurious arrow between two similarly contaminated roots or siblings. A score-aware transitive reduction then removes $i\to j$ when a two-hop path $i\to k\to j$ already exists with comparable bottleneck score. This second step is deliberately score-aware because biological or engineered systems can contain true shortcut edges; in such settings the pruning rule should be weakened or disabled unless the direct score is not stronger than the mediated path.

\begin{table}[h]
\centering
\caption{Focused \textsc{SKFM} Step~9 calibration on two six-node non-chain motifs. Both runs use free-vector fields with $\mathcal{L}_{\mathrm{div}}$ weight $5$, a rank-two spectral contamination subspace, $\alpha=4$, an ordered local parent extractor, root-pair masking, and score-aware two-hop transitive pruning. The latent penalty is weaker on collider/fork because its node-contamination scores are more tightly clustered.}
\label{tab:skfm_motif_step9_calibration}
\vspace{0.5em}
\resizebox{\textwidth}{!}{%
\begin{tabular}{lrrrrrl}
\toprule
\textbf{Motif} & \textbf{$\Delta_c$} & \textbf{$\lambda_{\mathrm{lat}}$} & \textbf{Max parents} & \textbf{Precision} & \textbf{Recall} & \textbf{Recovered arrows} \\
\midrule
Diamond & $\approx 0.25$ & $1.0$ & $2$ & $1.000$ & $1.000$ & $X_0{\to}X_1,\ X_0{\to}X_2,\ X_1{\to}X_3,\ X_2{\to}X_3,\ X_3{\to}X_4,\ X_3{\to}X_5$ \\
Collider/fork & $\approx 0.12$ & $0.5$ & $3$ & $1.000$ & $1.000$ & $X_0{\to}X_2,\ X_1{\to}X_2,\ X_2{\to}X_3,\ X_2{\to}X_4,\ X_3{\to}X_5,\ X_4{\to}X_5$ \\
\bottomrule
\end{tabular}
}
\end{table}

Table~\ref{tab:skfm_motif_step9_calibration} shows exact directed recovery on both the common-cause diamond and the common-effect collider/fork. The collider/fork run also reveals why a raw eigenvalue-gap rule is not sufficient for choosing $\lambda_{\mathrm{lat}}$: the eigengap suggested a stronger penalty, around $1.6$--$2.0$, which over-penalized real collider branches. The better signal is the contamination spread
\[
    \Delta_c=\max_i c_i-\min_i c_i.
\]
When $\Delta_c$ is large, as in the diamond where the root is geometrically cleaner than its descendants, a stronger latent penalty is safe. When $\Delta_c$ is small, as in collider/fork where roots and downstream branches are more similarly contaminated, the latent penalty must be weaker to avoid suppressing true arrows. The two motif runs are therefore consistent with an adaptive rule in which $\lambda_{\mathrm{lat}}$ increases with $\Delta_c$ and is clipped to a conservative range, rather than being set directly by the raw spectral eigengap.

These results sharpen the interpretation of \textsc{SKFM}: field calibration and graph extraction are separable. The continuity loss makes the learned fields trustworthy enough to use geometrically; contamination-gradient scoring separates causes from their contaminated effects; root-pair masking handles sibling/root confounding; and score-aware transitive pruning removes total-effect shortcuts. At the same time, subsequent ten-node random-DAG tests showed that forcing this direct Step~9 extractor to solve general graph recovery is unstable. Even with per-field heads and held-out continuity correlations above $0.8$, direct extraction by Jacobian influence, Radon--Nikodym variance reduction, sparse structure constants, root masks, and contamination-delta constraints varied from poor to moderate performance across seeds. This is not a failure of the geometric screen; it marks the boundary where local Lie-algebraic evidence should be handed to a global scorer.

\subsection{Experiment 5: Ten-Node Random DAGs with \textsc{Bridge}-Hybrid BIC}
We therefore tested the paper's main practical architecture on larger nonlinear DAGs: calibrated intervention fields produce a geometry-pruned candidate family, and a downstream decomposable score chooses the final parent sets. Each random graph has $d=10$ observed variables, sparse acyclic nonlinear structural equations, and three injected latent-confounded visible pairs. The calibrated field learner uses a shared encoder with per-node vector-field heads and the divergence-continuity loss above. This removed the field-interference failure observed with a single amortized intervention head: across the random-DAG runs, held-out continuity correlations were typically above $0.8$ and often above $0.9$.

The \textsc{Bridge}-hybrid extractor computes an autodifferentiated influence score
\[
    I_{ij}
    =
    \mathbb{E}_z
    \left[
    \left\|
    \frac{\partial v_j(z)}{\partial x_i}
    \right\|_2
    \right],
\]
optionally downweights bracket energy in a rank-two spectral residual subspace, thresholds the resulting score to retain roughly $40$--$50\%$ of the order-compatible arrows, and then runs a local Gaussian BIC scorer over parent sets inside the retained mask. In these diagnostics the true generator order is used for the ordered BIC step, so the experiment isolates candidate-family quality and local parent selection rather than CPDAG orientation.

\begin{table}[h]
\centering
\caption{\textsc{Bridge}-hybrid recovery on ten-node nonlinear random DAGs. The calibrated per-field heads generate the geometric mask; local BIC then scores parent sets inside the mask. The best row per seed is reported from a small grid over mask density, latent penalty, local polynomial score, maximum parents, and BIC penalty scale.}
\label{tab:bridge_hybrid_random10}
\vspace{0.5em}
\resizebox{\textwidth}{!}{%
\begin{tabular}{rrrrrrrr}
\toprule
\textbf{Seed} & \textbf{True edges} & \textbf{Mask arrows} & \textbf{Mask recall} & \textbf{Pred. edges} & \textbf{Precision} & \textbf{Recall} & \textbf{$F_1$} \\
\midrule
$21$ & $14$ & $22$ & $0.929$ & $16$ & $0.812$ & $0.929$ & $0.867$ \\
$22$ & $12$ & $18$ & $0.917$ & $13$ & $0.846$ & $0.917$ & $0.880$ \\
$23$ & $10$ & $22$ & $1.000$ & $12$ & $0.833$ & $1.000$ & $0.909$ \\
$24$ & $11$ & $18$ & $1.000$ & $14$ & $0.786$ & $1.000$ & $0.880$ \\
$25$ & $14$ & $22$ & $0.929$ & $15$ & $0.867$ & $0.929$ & $0.897$ \\
$26$ & $9$  & $18$ & $0.667$ & $7$  & $0.857$ & $0.667$ & $0.750$ \\
\midrule
\textbf{Mean} & -- & $20.0$ & $0.907$ & $12.8$ & $0.834$ & $0.907$ & $0.864$ \\
\bottomrule
\end{tabular}
}
\end{table}

Table~\ref{tab:bridge_hybrid_random10} is the main scalability result for the present implementation. The same ten-node seeds defeated direct \textsc{SKFM} extraction: density-ratio selectors, integrated-flow Radon--Nikodym weights, sparse structure-constant selectors, and root/contamination hard masks all became topology-sensitive. In contrast, the \textsc{Bridge}-hybrid pipeline is stable across the sweep, with mean directed $F_1=0.864$ and mean precision $0.834$. The hardest case, seed~26, has four independent roots and lower mask recall; even there, the downstream BIC scorer keeps precision high while recall is limited by the screen. This validates the intended division of labor: Lie/Frobenius geometry should reduce and regularize the candidate family, while a score-based layer should handle global acyclicity, high indegree, shortcut edges, and multi-root ambiguity.

Appendix~\ref{app:random-dag-visuals} shows representative recovered graphs from this sweep. Those visual diagnostics are included to make the target explicit: left panels show the sampled visible direct-effect DAG, right panels show the directed graph selected inside the \textsc{Bridge} mask, with true positives, false positives, and missed arrows colored separately. The latent-confounded extension overlays hidden common causes as dashed bidirected arcs, matching the ADMG convention, while keeping the scored directed graph restricted to visible arrows.

\subsection{Experiment 6: Sachs \textsc{Bridge}-Pruned GES and Calibration Diagnostic}
We next evaluate the same geometric candidate-screening principle on the Sachs protein-signaling benchmark. The dataset contains $n=853$ measurements over $d=11$ signaling proteins, with a canonical directed reference graph containing $17$ arrows. Unlike the synthetic experiment above, Sachs is a noisy biological benchmark rather than a controlled data-generating process. Protein signaling is also closer to a rapid equilibrium system than to a clean acyclic transport model, so we treat Sachs as a real-data diagnostic: can the same geometry run on biological measurements, and where does calibration begin to fail?

The experiment compares four families of estimators:
\begin{enumerate}
    \item \textbf{Pooled GES baseline:} run unconstrained GES with a Gaussian BIC score on the full Sachs data.
    \item \textbf{\textsc{Bridge}-pruned GES/BIC:} first compute the Lie/Frobenius candidate mask, then run a GES-style forward/backward/reversal search only over mask-admissible arrows.
    \item \textbf{\textsc{Bridge}-pruned TCES/GES:} augment the pruned GES score with environment-specific likelihood and $j$-stability terms. Because the available Sachs panel in the repository contains generated environment labels, we regenerate a $k=5$ panel and treat this condition as a lightweight interventional-scoring ablation rather than a claim about the original experimental perturbation design.
    \item \textbf{Per-field \textsc{Bridge}-hybrid BIC diagnostic:} learn per-node intervention-field heads with the divergence-continuity loss, build an autodifferentiated geometric mask, and score local parent sets by BIC inside that mask. For this diagnostic we use a canonical topological order derived from the Sachs reference graph, so the reported number measures field calibration and parent-set scoring rather than order discovery.
\end{enumerate}

This use of GES also has a deeper categorical interpretation. In earlier work on the higher algebraic K-theory of causality, covered edge reversals in Chickering-style equivalence classes were lifted from graph moves to natural transformations between homotopically equivalent cPROP functors, with Grothendieck group completion organizing causal equivalence at the categorical level \citep{sridhar:higher-algebraic-k-theory}. The present information-geometric layer can be read as an analytic front end to that categorical picture: Lie/Frobenius screening proposes the local tangent directions and equivalence moves worth scoring, while GES-style reversal search navigates the reduced family of categorical causal models.

The original kernel \textsc{Bridge} screen retains $44$ directed candidate arrows out of the $110$ possible non-self arrows and preserves $12/17$ canonical Sachs arrows. More importantly for the tractability thesis, it contains $15/16$ arrows returned by pooled GES. Thus the Lie/Frobenius mask removes $60\%$ of the directed candidate space while preserving nearly the entire GES-selected structure. The per-field diagnostic is more conservative: its best high-precision mask and BIC scorer returns a $9$-edge graph with directed precision $0.889$ but recall $0.471$. The associated field calibration is substantially weaker than in the synthetic ten-node runs, with mean held-out continuity correlation $0.587$ and minimum correlation $0.499$. We therefore report Sachs as a calibration frontier rather than as a solved biological discovery benchmark.

\begin{table}[h]
\centering
\caption{Sachs \textsc{Bridge}-pruned protein-signaling benchmark and per-field calibration diagnostic. Directed metrics compare adjacency matrices to the canonical Sachs graph. Skeleton metrics ignore orientation. The per-field row uses a reference topological order and should be read as a calibration/parent-set diagnostic rather than a fully order-free real-data recovery claim.}
\label{tab:sachs_geometry_ges}
\vspace{0.5em}
\resizebox{\textwidth}{!}{%
\begin{tabular}{lrrrrr}
\toprule
\textbf{Method} & \textbf{Edges} & \textbf{Directed $F_1$} & \textbf{Directed Precision} & \textbf{Directed Recall} & \textbf{Skeleton $F_1$} \\
\midrule
Pooled GES & $16$ & $0.485$ & $0.500$ & $0.471$ & $0.640$ \\
Per-env GES intersection & $8$ & $0.320$ & $0.500$ & $0.235$ & $0.381$ \\
\textsc{Bridge}-pruned GES/BIC & $8$ & $0.320$ & $0.500$ & $0.235$ & $0.640$ \\
\textsc{Bridge}-pruned TCES/GES ($k=5$ env) & $9$ & $0.385$ & $0.556$ & $0.294$ & $0.615$ \\
Per-field \textsc{Bridge}-hybrid BIC (order diagnostic) & $9$ & $0.615$ & $0.889$ & $0.471$ & $0.615$ \\
\bottomrule
\end{tabular}
}
\end{table}

Table~\ref{tab:sachs_geometry_ges} shows that pooled GES obtains the strongest unconstrained directed baseline among the order-free runs. However, the CausalLearn pooled GES output includes several reciprocal arrows, reflecting partially directed or ambiguous CPDAG structure under the adjacency extraction used in this experiment. On skeleton recovery, \textsc{Bridge}-pruned GES/BIC matches pooled GES exactly ($F_1=0.640$) while operating inside the $44$-arrow \textsc{Bridge} mask. Adding the TCES-style environment and $j$-stability score improves directed recovery within the pruned family from $F_1=0.320$ to $F_1=0.385$, with precision increasing from $0.500$ to $0.556$. The per-field diagnostic obtains the highest directed precision and $F_1$ in the table, but it uses the reference order and has low recall because the learned real-data field calibration is not yet at the synthetic level.

This benchmark sharpens the practical message of the random-DAG experiment. The information-geometric layer need not replace GES; it can act as a geometric front end that filters the model class before score-based search. On synthetic graphs where continuity calibration succeeds, this hybrid produces high directed recovery. On Sachs, the method still runs and can produce a precise parent-set diagnostic, but the continuity score exposes weak overlap, unmodeled contexts, possible feedback, or other biological nonstationarity. Stronger Wasserstein regularization, domain adaptation, and CPDAG-aware orientation scoring are therefore future work rather than hidden assumptions in the present claim.

We also ran \textsc{SKFM} Algorithm~\ref{alg:skfm} on Sachs as a real-data pilot using an inferred net-flow order rather than the CSV column order. The learned amortized fields closely matched the kernel geometry, with influence correlation $0.947$ and relative influence MSE $0.103$, and the spectral factorization selected a rank-two residual footprint. However, direct graph extraction remained weak against the canonical Sachs DAG: the $13$-edge graph achieved directed $F_1=0.133$ with precision $0.154$ and recall $0.118$. This is a useful negative control for \textsc{SKFM}: the field-learning layer transfers to Sachs, but robust real-data graph extraction still benefits from the downstream score-based layer used in Table~\ref{tab:sachs_geometry_ges}.

\subsection{Experiment 7: S9 Gene-Expression Pilot}
The repository also contains a prepared S9 expression panel derived from the raw Science supplement table. Unlike Sachs, this panel is not accompanied by a checked-in biological gold DAG. We therefore use S9 as a qualitative real-data stress test for the computational promise of the method: does the Lie/Frobenius screen turn an otherwise enormous score-based model class into a small enough family for local GES-style search, without pretending that we can report SHD or $F_1$ against an unrelated reference graph?

For interpretability, we run the pilot on the nine gene-expression variables in the prepared panel and exclude sequencing-count technical covariates by default. The panel contains $n=33$ samples and two inferred environments with counts $30$ and $3$, so the present run uses pooled \textsc{Bridge}-pruned BIC/GES rather than making a strong interventional-scoring claim. With a stricter two-parent-per-target \textsc{Bridge} screen, the candidate directed arrows drop from $72$ to $18$.

\begin{table}[h]
\centering
\caption{\textsc{Bridge} S9 gene-expression pilot. Because no validated S9 gold DAG is checked into the repository, this table reports search-space compression and selected qualitative edges rather than accuracy metrics.}
\label{tab:s9_gene_pilot}
\vspace{0.5em}
\begin{tabular}{lr}
\toprule
\textbf{Quantity} & \textbf{Value} \\
\midrule
Samples / genes & $33 / 9$ \\
Possible directed non-self arrows & $72$ \\
\textsc{Bridge}-retained candidate arrows & $18$ \\
All labeled DAGs on nine genes & $1{,}213{,}442{,}454{,}842{,}881$ \\
Full-space MEC scale (DAGs$/3.7$ heuristic) & $\approx 3.28{\times}10^{14}$ \\
Acyclic DAGs inside geometry mask & $119{,}216$ \\
Exact acyclic-family compression & $10{,}178{,}532{,}704\times$ \\
Greedy pruned-GES moves evaluated & $73$ \\
Selected directed edges & $3$ \\
\bottomrule
\end{tabular}
\end{table}

The selected qualitative graph contains the arrows
\[
    \texttt{C21orf82}\to\texttt{B2M},\qquad
    \texttt{RPS9}\to\texttt{B2M},\qquad
    \texttt{S100A5}\to\texttt{RPS9}.
\]
The largest Lie-bracket anisotropy pairs are \texttt{C21orf82--PSAT1}, \texttt{C21orf82--B2M}, and \texttt{OLFM2--S100A5}, which should be read as latent-confounding or non-closure candidates rather than validated biological mechanisms. This experiment is deliberately modest, but it is valuable: on a real expression panel, the geometric layer converts a quadrillion-scale labeled DAG family, and still a roughly $10^{14}$-scale Markov-equivalence-class family under the $3.7$ quotient heuristic, into an exactly countable acyclic candidate family before score-based search.

The \textsc{SKFM} real-data pilot gives a complementary discovery-only view on the same nine-gene panel. With an inferred net-flow order and the top-$k$/quantile extractor, the learned fields again track the kernel geometry well, with influence correlation $0.920$ and relative influence MSE $0.089$. The spectral factorization selects a rank-two residual footprint, and the direct extractor returns a compact $10$-edge graph. Because S9 has no validated gold DAG in this repository, we report these as field-quality and discovery-compression diagnostics rather than accuracy metrics.

\subsection{Experiment 8: \textsc{Bridge}-Constrained DCDI}
Finally, we ask whether the same Lie/Frobenius candidate family can regularize a differentiable graph-learning objective. DCDI is a natural comparison point because it already represents causal discovery as continuous optimization over graph parameters. The question here is therefore not whether DCDI should replace the geometric screen, but whether DCDI's learned edge weights become more informative when evaluated inside the admissible family selected by Lie/Frobenius closure.

The resulting \textsc{Bridge}-DCDI variant combines two complementary signals. DCDI supplies differentiable edge scores learned from the observed variables, while the Lie/Frobenius module supplies a directed admissibility mask together with local influence and bracket-anisotropy diagnostics. We then select high-scoring edges only inside this \textsc{Bridge} candidate family, optionally enforcing acyclicity. This turns differentiable graph learning into another scoring rule over the same \textsc{Bridge}-pruned family used in the earlier GES/TCES experiments.

Table~\ref{tab:lie_dcdi_results} reports three representative runs of this integrated pipeline. The PISA 2022 run uses the real OECD trend variables \texttt{escs\_trend}, \texttt{hisei\_trend}, \texttt{homepos\_trend}, and \texttt{paredint\_trend}; because no ground-truth DAG is available for these variables, it is reported as a discovery-only comparison against the Kan-Do baseline graph. The S9/Sachs and sheaf-fast runs use reference adjacency matrices and therefore report directed and skeleton metrics.

\begin{table}[h]
\centering
\caption{\textsc{Bridge}-constrained DCDI results. PISA is discovery-only; the two ground-truth rows report directed and skeleton metrics. ``Base $F_1$'' is the Kan-Do baseline for the same benchmark.}
\label{tab:lie_dcdi_results}
\vspace{0.5em}
\resizebox{\textwidth}{!}{%
\begin{tabular}{lrrrrrrr}
\toprule
\textbf{Dataset} & \textbf{Vars} & \textbf{\textsc{Bridge} Cand.} & \textbf{Selected Edges} & \textbf{Directed $F_1$} & \textbf{SHD} & \textbf{Skeleton $F_1$} & \textbf{Base $F_1$} \\
\midrule
PISA 2022 OECD trends & $4$ & $7$ & $4$ & -- & -- & -- & -- \\
S9/Sachs signaling & $11$ & $22$ & $15$ & $0.296$ & $19$ & $0.400$ & $0.118$ \\
Sheaf-fast chain & $8$ & $16$ & $9$ & $\mathbf{0.500}$ & $8$ & $\mathbf{0.875}$ & $0.333$ \\
\bottomrule
\end{tabular}
}
\end{table}

On the real PISA 2022 OECD trend table, Lie-DCDI selects four arrows:
\[
\begin{aligned}
\texttt{escs\_trend}&\to\texttt{hisei\_trend},&
\texttt{escs\_trend}&\to\texttt{homepos\_trend},\\
\texttt{escs\_trend}&\to\texttt{paredint\_trend},&
\texttt{paredint\_trend}&\to\texttt{hisei\_trend}.
\end{aligned}
\]
This graph shares both Kan-Do baseline PISA arrows, giving directed Jaccard overlap $0.5$ against that comparison graph. We do not report causal accuracy on PISA because no gold DAG is available for these OECD trend variables.

The \textsc{SKFM} Algorithm~\ref{alg:skfm} pilot on the same PISA trend table gives an additional discovery-only check. With an inferred net-flow order and the top-$k$/quantile extractor, the baseline geometry screen contains $9$ candidate arrows and the direct \textsc{SKFM} extractor returns a compact $4$-edge graph. The learned amortized fields match the kernel geometry very closely, with influence correlation $0.993$ and relative influence MSE $0.003$, and the spectral factorization selects a rank-two residual footprint. As with the Lie-DCDI PISA row, these numbers should be read as field-quality and structural-compression diagnostics rather than as causal accuracy claims.

On the S9/Sachs signaling benchmark, Lie-DCDI recovers four directed ground-truth edges:
\[
\begin{aligned}
\texttt{Raf}&\to\texttt{Mek},&
\texttt{Plcg}&\to\texttt{PIP3},\\
\texttt{PIP2}&\to\texttt{Plcg},&
\texttt{PIP3}&\to\texttt{Akt}.
\end{aligned}
\]
Against the S9 reference graph, this gives precision $0.267$, recall $0.333$, directed $F_1=0.296$, and $\mathrm{SHD}=19$. This improves the Kan-Do S9 baseline, whose directed $F_1$ is $0.118$ and $\mathrm{SHD}=30$. A caveat is that the S9 reference graph contains a reciprocal pair, whereas this Lie-DCDI run enforces a DAG; hence some reference edges are structurally unrecoverable under this setting.

On the sheaf-fast chain, Lie-DCDI recovers four of seven directed chain arrows and all seven true skeleton adjacencies, with two extra skeleton edges. Thus directed $F_1$ rises to $0.500$ compared with the sheaf-fast Kan-Do baseline of $0.333$, while skeleton $F_1$ reaches $0.875$. This run uses the same eight-node chain model as the earlier sheaf-fast experiment.

For context, we also include an earlier synthetic DCDI comparison from the same experimental suite. On linear ``perfect intervention'' synthetic DAGs over three settings, a $j$-stable DCDI aggregation path improves vanilla DCDI from directed $F_1=0.289$ to $0.429$, while reducing directed SHD from $76.23$ to $32.37$ and mean selected edges from $186.67$ to $55.40$. The geometry-constrained DCDI results sharpen the same lesson in a more direct way: differentiable edge weights are most useful when passed through a structural screen that encodes regime stability, local influence, and Lie/Frobenius closure.

\section{Identification Limits and Falsification Protocol}
\label{sec:limitations}

The experiments estimate selected DAGs for comparison with known synthetic generators, but latent-confounded discovery is generally an equivalence-class problem. In observational settings the identifiable object is normally a PAG representing MAG invariants. With known interventions, additional endpoint marks can be oriented only under the relevant intervention design and assumptions. With unknown soft targets, the appropriate target is a $\Psi$-PAG over graph--target pairs \citep{jaber2020soft}. None of the local \textsc{Bridge} scores changes these facts.

The current empirical evidence is strongest for \textsc{Bridge}-KT screening. The direct observational--interventional gap is an essential matched baseline because it can detect many of the same regime differences without derivatives or brackets. Conditional-invariance tests, density-ratio effective sample size, held-out continuity error, random tangent fields, permuted regime labels, and sensitivity to adjustment coordinates are likewise required to determine whether bracket geometry contributes incremental information. A null direct gap does not prove absence of an effect because exact or near cancellation is possible; conversely, a large gap does not identify a direct edge.

For unknown-target data, future evaluations should compare with $\Psi$-FCI on $\Psi$-PAG endpoint recovery rather than DAG $F_1$. For context-indexed domains, JCI-style baselines and explicit assumptions about context variables are required. For known targets with latent variables, FCI/RFCI or interventional PAG methods provide matched equivalence-class baselines. These comparisons are part of the due-diligence agenda and are not silently replaced by the current GES/BIC experiments.

The \textsc{SKFM} ablations isolate a second limitation. The learned fields can correlate strongly with kernel teachers while the final graph is poor. On non-chain motifs, supplying the true or coordinate order substantially improves extraction; on harder random DAGs, spectral, stochastic, and structure-constant extractors do not reliably overcome a bad learned order. Thus field calibration, candidate retention, ordering, and final graph selection must be reported separately. Exact motif recovery under a supplied order is a useful mechanism check, not evidence for general order-free identification.

\section{Scaling the Geometric Screen}
\label{sec:scaling}
\subsection{Amortized Flow Architecture for High-Dimensional Scalability}
A fundamental limitation when scaling the geometric screen to hundreds or thousands of nodes is the computational burden of training $d$ independent Continuous Normalizing Flows (CNFs), one for every observed variable. To bypass this $\mathcal{O}(d)$ architectural scaling bottleneck, we use an \textbf{Amortized Flow Architecture}. Instead of optimizing separate models, a single neural vector field parameterizes all regime responses by embedding variable indices into a shared continuous context space.

\subsubsection{Amortized Vector Field Parameterization}
Let $X = [X_1, X_2, \dots, X_d]^T \in \R^d$ be the vector of all visible variables. To evaluate interventions or conditional shifts on any arbitrary variable index $i \in \{1, \dots, d\}$, we define an index embedding function $\phi: \{1, \dots, d\} \to \R^k$ that maps discrete nodes to a compact, continuous vector space. 

Instead of independent neural velocity fields, we instantiate a single amortized neural network $f_\Theta$ parameterized by global weights $\Theta$. The continuous dynamics of the system are governed by a conditional Ordinary Differential Equation (ODE):
\begin{equation}
    \frac{dz(t)}{dt} = f_\Theta\big(z(t), t, \phi(i), \mathbf{m}_i\big)
\end{equation}
where $z(0) \sim p_0$ is a shared isotropic Gaussian base distribution, $\phi(i)$ is the target variable context embedding, and $\mathbf{m}_i \in \{0, 1\}^d$ is a regime mask indicating whether node $i$ is undergoing an active manipulation ($\mathbf{m}_i = 1$) or passive observation ($\mathbf{m}_i = 0$).

\subsubsection{Pointwise Radon--Nikodym Estimation via Shared Jacobians}
By integrating the single amortized vector field forward from $t=0$ to $t=1$, the change-of-variables formula tracks both the observational manifold ($p_{\text{obs}}$) and any target interventional manifold ($p_{\text{do}(X_i)}$) via a single shared forward pass. The log-likelihood profiles are computed natively via the instantaneous change-of-variables equation:
\begin{align}
    \log p_{\text{obs}}(x) &= \log p_0(z(0)) - \int_0^1 \text{Tr}\left( \nabla_{z(t)} f_\Theta\big(z(t), t, \phi(i), \mathbf{m}_i = \mathbf{0}\big) \right) dt \\
    \log p_{\text{do}(X_i)}(x) &= \log p_0(z(0)) - \int_0^1 \text{Tr}\left( \nabla_{z(t)} f_\Theta\big(z(t), t, \phi(i), \mathbf{m}_i = \mathbf{1}\big) \right) dt
\end{align}
Parameter sharing does not by itself cancel the base-density terms: the observational and interventional inverse trajectories can map the same evaluation point $x$ to different base coordinates. When both fitted laws have normalized likelihoods and $p_{\mathrm{do}(X_i)}\ll p_{\mathrm{obs}}$, the ratio is computed from the two complete log likelihoods:
\begin{equation}
    \rho_i(x)
    =
    \exp\!\left(
      \log p_{\mathrm{do}(X_i)}(x)-\log p_{\mathrm{obs}}(x)
    \right).
\end{equation}
Thus amortization shares parameters across regimes, while density-ratio validity still requires likelihood calibration, absolute continuity, and overlap.

\subsubsection{Amortized Lie Bracket Computation}
The amortization of the vector fields drastically simplifies the evaluation of the Lie bracket grid matrix for high-dimensional discovery. Let $v_i(x) = f_\Theta(x, 1, \phi(i), \mathbf{1})$ and $v_j(x) = f_\Theta(x, 1, \phi(j), \mathbf{1})$ be the velocity vectors at the boundary $t=1$. The Lie bracket $[v_i, v_j]$ is calculated across all pairs using the automated gradients of the shared network:
\begin{equation}
    [v_i, v_j] = \left( \frac{\partial f_\Theta}{\partial x} \Big|_{\phi(j)} \right) \cdot f_\Theta\big(x, 1, \phi(i), \mathbf{1}\big) - \left( \frac{\partial f_\Theta}{\partial x} \Big|_{\phi(i)} \right) \cdot f_\Theta\big(x, 1, \phi(j), \mathbf{1}\big)
\end{equation}
By computing these spatial jacobians sequentially using modern parallel automatic differentiation engines (such as vmap in PyTorch), the computation of the full $d \times d$ Lie bracket residual matrix scales at $\mathcal{O}(d^2)$ vector-matrix products over a single neural model, completely avoiding the prohibitive cost of training and coordinating hundreds of individual networks.

\subsection{Weak Overlap and Transport Regularization}
When observational and interventional regimes have weak overlap, pointwise density ratios can have high variance or fail to exist on parts of the interventional support. A Wasserstein or flow-matching penalty can regularize the learned transport numerically, but it does not restore absolute continuity, identify an intervention target, or justify a causal interpretation by itself. We therefore report overlap diagnostics and effective sample sizes, calibrate transport fields on held-out data, and treat optimal-transport losses as stabilization devices rather than identification theorems.

\subsection{Spectral Summaries of Residual Nonclosure}
For scalable screening, let $b_i(x)$ denote a vectorized collection of projected closure-residual features associated with intervention index $i$. The positive-semidefinite matrix
\[
  G_{ij}=\mathbb E\langle b_i(X),b_j(X)\rangle
\]
can be approximated by randomized eigensolvers or pair subsampling. Its leading eigenspace is a compact summary of correlated visible nonclosure. Proposition~\ref{latent_prop} gives the additional assumptions under which its rank can be interpreted as the dimension of a latent loading footprint. Without those assumptions, the spectrum is only a residual diagnostic: it need not count hidden variables, identify their coordinates, or provide a valid adjustment set.

The resulting projector may be used to rescue or downweight candidate arrows, but not to ``peel away'' confounding exactly. Any improvement must therefore be established empirically against direct regime-gap controls and latent-aware downstream baselines.

\subsection{Scalability Ablation: Spectral Updates and Frobenius Pair Budgets}
\label{sec:scalability-ablation}

The implementation also exposes where the Lie-algebraic discovery layer spends its time. A reproducible runtime ablation decomposes the \textsc{SKFM}/\textsc{Bridge} path into local field estimation, pairwise Lie/Frobenius geometry, optional \textsc{SKFM} field training and Lie penalty, spectral curvature factorization, and final extraction. With exact all-pairs geometry, the spectral stage dominated the small scaling sweep; replacing the dense eigensolve and full curvature pass by a randomized rank-$3$ update with an eight-pair curvature budget reduced the spectral stage at $d=8$ from about $0.518$ seconds to $0.149$ seconds, a $3.48\times$ stage speedup. After this change, pairwise Frobenius geometry became the next bottleneck. Restricting Frobenius residuals to candidate pairs touched by the directed influence screen reduced the $d=8$ Frobenius stage from about $0.267$ seconds to $0.120$ seconds, a further $2.23\times$ stage speedup. With both approximations enabled, the measured $d=8$ path drops to about $0.360$ seconds in the maintained benchmark.

\begin{figure}[t]
\centering
\includegraphics[width=0.72\linewidth]{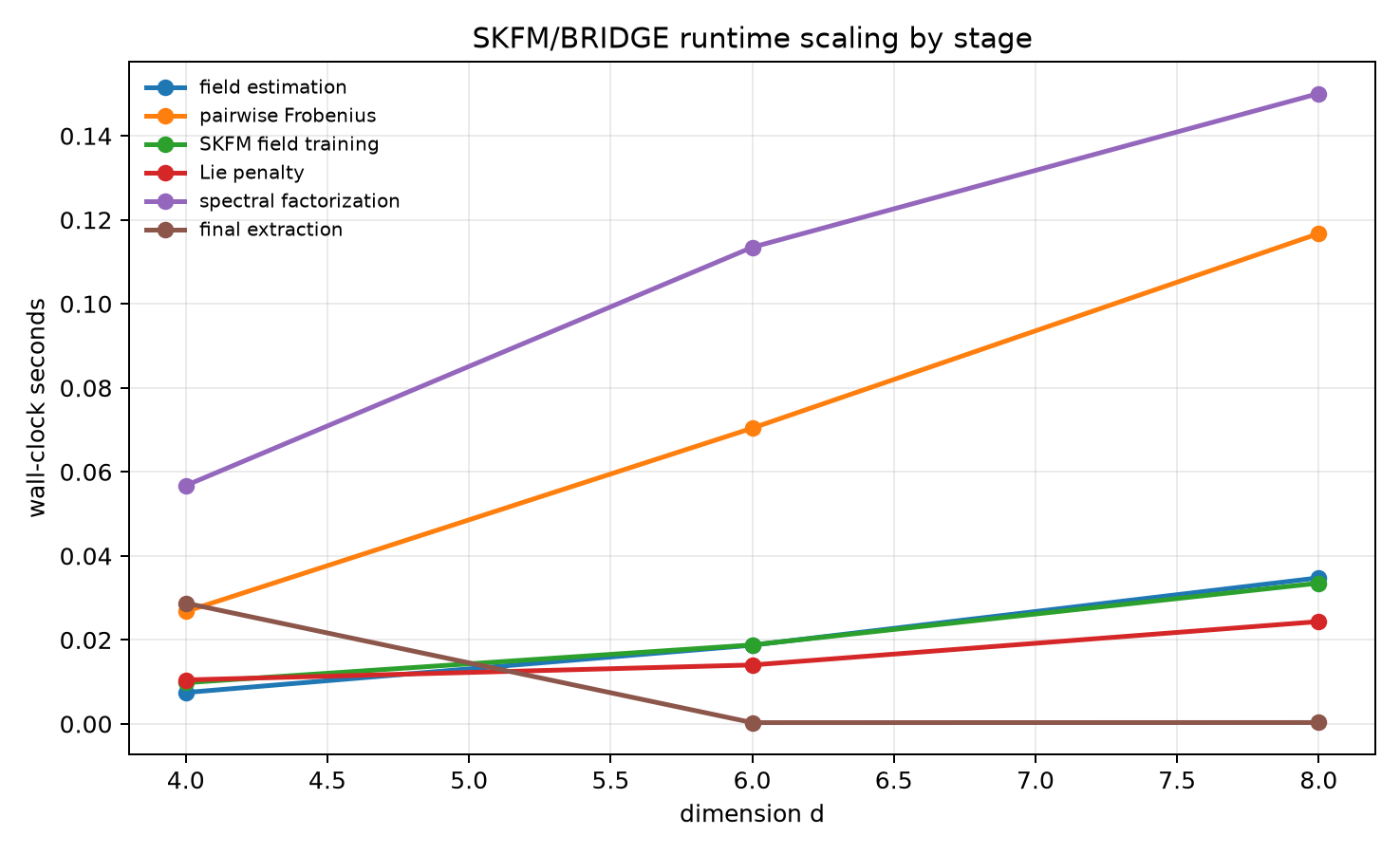}
\caption{Stage-wise runtime of the optimized \textsc{SKFM}/\textsc{Bridge} path at fixed sample size $n=64$, using randomized spectral updates and candidate-pair Frobenius geometry. After the randomized spectral update, the pairwise geometry stage becomes visible as the next bottleneck; after candidate-pair Frobenius, no single Lie-algebraic subroutine dominates at these dimensions.}
\label{fig:skfm_bridge_runtime_scaling}
\end{figure}

The same ablation makes the tractability--expressivity tradeoff measurable. On the maintained six-seed ten-node random-DAG demo, exact all-pairs Frobenius geometry remains the quality default: a fresh rerun gives mean directed $F_1=0.904$, precision $0.929$, and recall $0.885$. Candidate-pair Frobenius preserves mean recall and screen recall on this demo, but the partial anisotropy matrix changes spectral rescue enough to introduce additional false positives, reducing mean $F_1$ to $0.887$. A sampled 30-pair middle ground reaches mean $F_1=0.890$. Thus the adaptive modes are useful scaling controls, but they should be treated as calibrated geometry-budget approximations rather than replacements for the exact screen when maximum precision is required.

\begin{table}[t]
\centering
\small
\begin{tabular}{lccccc}
\toprule
\textbf{Frobenius mode} & \textbf{Precision} & \textbf{Recall} & $\boldsymbol{F_1}$ & \textbf{Screen recall} & \textbf{SHD} \\
\midrule
Exact all pairs & 0.929 & 0.885 & 0.904 & 0.885 & 2.33 \\
Candidate pairs & 0.890 & 0.885 & 0.887 & 0.885 & 2.83 \\
Candidate + sampled pairs & 0.901 & 0.881 & 0.890 & 0.881 & 2.67 \\
\bottomrule
\end{tabular}
\caption{Quality impact of reducing the Frobenius pair budget on the standard six-seed ten-node \textsc{Bridge}-hybrid random-DAG demo. Candidate-pair geometry preserves recall but loses precision through additional spectral-rescue false positives. This illustrates the computational tractability--expressivity tradeoff induced by approximating the full Lie/Frobenius geometry.}
\label{tab:frobenius_budget_quality}
\end{table}

\section{Summary  and Future Research Directions}
\label{sec:conclusion}

This paper develops \textsc{Bridge}, an information-geometric front end for causal discovery from observational and heterogeneous regime data, together with \textsc{SKFM}, an amortized field-learning and residual-summarization extension. The principal claims concern calibrated response fields, high-recall candidate screening, residual closure diagnostics, and order-dependent direct extraction.

\begin{itemize}
    \item \textbf{Visible nonclosure} is summarized by projected Lie-bracket residuals. Latent confounding is one possible cause, but the diagnostic is deliberately nonspecific.
    \item \textbf{Discovery} proceeds by retaining a candidate family and delegating final identification to a regime-appropriate learner; direct \textsc{SKFM} extraction is conditional on a reliable variable order.
\end{itemize}

The response-field representation provides a smooth interface for comparing observational and interventional regimes. Fisher geometry supplies a useful invariant metric under its standard statistical assumptions, while the causal interpretation continues to come from the data regime, intervention design, and downstream identification assumptions.

Causal discovery in \textsc{Bridge} uses localized, parallelizable checks of vector-field influence and closure followed by a conventional scorer on a smaller family. In the first practical benchmark, this geometric screening reduced a seven-node acyclic search family by more than four orders of magnitude while retaining all generator edges; under the harder endpoint-confounded version, the downstream TCES/BIC scorer recovered most but not all directed chain structure. A latent-aware FCI baseline on the same chain sweep recovered the true skeleton with one extra adjacency at each size. A non-chain motif stress test showed both the limitation of a pure influence mask and the potential value of spectral residual footprints. The companion \textsc{SKFM} ablation further showed that strong field calibration does not by itself guarantee reliable graph extraction.

The decisive scalability check is the ten-node nonlinear random-DAG experiment. Direct \textsc{SKFM} graph extraction became unstable across these seeds even when the learned fields were well calibrated, confirming that global DAG recovery from continuous Lie-space scores is ill-conditioned without a downstream search layer. The \textsc{Bridge}-hybrid pipeline resolved this: per-field calibrated intervention heads produced a high-quality geometric candidate family, and local BIC scoring inside that family achieved mean directed $F_1=0.864$ with mean precision $0.834$. This is the practical contribution of the present implementation. In Sachs protein signaling, by contrast, the same per-field calibration reached only mean continuity correlation $0.587$, and the best ordered parent-set diagnostic achieved high precision but low recall. We therefore treat Sachs as a real-data calibration frontier rather than as a solved benchmark. The S9 expression pilot shows the same compression phenomenon on real data without a gold DAG, reducing a nine-gene quadrillion-scale labeled DAG family to $119{,}216$ acyclic candidates. The geometry-constrained DCDI experiment shows the same structural lesson in a differentiable-discovery setting: DCDI weights inside a Lie/Frobenius candidate family improve Kan-Do baselines on S9/Sachs and sheaf-fast ground-truth benchmarks.

Taken together, Proposition~\ref{geom-screen}, Corollary~\ref{downstream-selection}, and Propositions~\ref{latent_prop} and~\ref{gradient_flow} delimit the mathematical spine of this empirical story. \textsc{Bridge} has a conditional screen-retention guarantee; any downstream identification guarantee is inherited from the chosen learner. \textsc{SKFM} estimates a residual-footprint rank only under an explicit factorization and supplies acyclicity only relative to an order. Finite-sample rates, specificity for latent confounding, robust biological transport calibration, and order-free graph extraction remain open.

\subsection{Statistical Efficiency and Finite-Sample Analysis}

While Propositions~\ref{geom-screen} and~\ref{latent_prop} give asymptotic screening and rank-recovery guarantees under margin, smoothness, and eigen-gap assumptions, finite-sample performance remains an open area of investigation. The following points outline key challenges for future \textsc{SKFM} implementations:

\begin{itemize} 

\item \textbf{Convergence Rates}: Establishing rates for residual-Gram estimation and screen retention is needed to complement the fixed-dimensional consistency statements.

\item \textbf{Generalization and Noise Robustness}: The paper assumes the statistical manifold is smooth, but finite samples may introduce noise that obscures curvature. Future work should explore how the Wasserstein regularizer stabilizes spectral factorization in small samples where manifold curvature is not well-separated.

\item \textbf{Downstream Performance}: The ten-node random-DAG experiments show that, when per-field continuity calibration succeeds, geometric compression can directly improve downstream recovery rather than merely reducing enumeration cost. The remaining question is how broadly this transfers across graph families, sample sizes, and real-data regimes where overlap and stationarity are weaker.

\item \textbf{Comparison with LiNGAM}: Future work should compare against LiNGAM-family methods only on benchmarks satisfying their functional and distributional assumptions, while separately reporting mask retention and final graph recovery.
These questions define the path from theoretical guarantees to scalable, general-purpose causal discovery tools.

\end{itemize} 

\subsection{Natural-Gradient Causal Discovery}

Once response-field learning is parameterized on a statistical model, the Euclidean gradient need not be the best search direction. The Fisher metric defines a natural gradient flow over model parameters under the usual information-geometric assumptions \citep{amari1998natural}. Natural-gradient methods can be viewed as mirror-descent updates in dual information-geometric coordinates \citep{beck2003mirror,DBLP:conf/nips/ThomasDGM13}; in the present setting, closure and calibration losses can constrain admissible updates. This is an optimization proposal, not an additional causal-identification result.

\subsection{Learning Causality from Language} 

A second application frontier is causal discovery from text, where the ``interventions'' are often semantic or institutional perturbations rather than laboratory do-operations.  In settings such as 10-K filings, scientific papers, policy documents, or product reviews, large language models can extract local causal claims, but those claims must then be organized into persistent world models rather than a flat list of triples. This connects the present framework to DEMOCRITUS-style large causal models from language \citep{mahadevan2025largecausalmodelslarge} and PROMETHEUS-style causal atlases over text, data, and models \citep{mahadevan2026prometheus}. From the present perspective, each textual neighborhood defines a local predictive-state or sheaf patch \citep{singh2004predictive,maclane_sheaves}; semantic interventions induce vector fields on the distribution of claims, entities, risks, and outcomes; and non-closing Lie brackets reveal hidden context, contradictory evidence, or latent sector/customer/product structure. This suggests a path from information-geometric causal discovery on tabular interventions to topos world models in which causal knowledge extracted from language is scored, glued, and stress-tested geometrically.

Initial integrations with a system for building trustworthy foundries suggest this direction is empirically testable \citep{mahadevan2026odyssey}. Preliminary proxy-field studies on MyFixIt laptop repair, Adobe 10-K filing retrieval, Democritus GLP-1 causal extraction, Lovesac product-feedback review, and MCT argumentative microtexts show how BRIDGE/SKFM residuals can act as vocabulary critics for PSR failures, missing qualifiers, overclaims, hidden mechanisms, and rebuttal/undercut structure. These examples are not yet full observational \textsc{SKFM} training results, but they indicate a practical route for converting language-derived world-model failures into the same residual-analysis interface used by the geometric method.

\subsection{Manifold Structure} 
  
  When an interventional regime pushes a system into state-space sectors completely unvisited under pure observation, the Radon--Nikodym density ratio $\rho(z)$ encounters severe support divergence. Future work must investigate regularized flow boundaries or incorporate optimal transport (Wasserstein) metrics \citep{villani2009optimal} to gracefully bound the metric tensor as overlap approaches zero.

  \begin{itemize} 
    \item \textbf{Extension to Non-Smooth, Discrete Spaces:} The present method relies on smooth vector fields. Discrete analogues require finite-difference or other nonsmooth constructions and new calibration theory.
    \item \textbf{Transport Regularization:} Benamou--Brenier or related optimal-transport objectives may stabilize field estimation \citep{benamou2000computational}, but their effect on identification and overlap diagnostics must be evaluated separately.
\end{itemize}

The resulting framework is best viewed as a geometric screening and diagnostic layer: it organizes regime responses continuously, measures failures of visible closure, and hands a reduced candidate family to causal procedures whose assumptions match the available data.

\section*{Changes in This Version}
This revision narrows and corrects the paper's theoretical claims while preserving the reported experiments. It removes the proposed Kan-extension interpretation of conditioning, intervention, front-door adjustment, and counterfactual transport; distinguishes fixed structural interventions from learned distribution-transport fields; and treats projected Lie brackets as nonspecific closure diagnostics rather than certificates of latent confounding. The geometry theorem is restated as a candidate-retention result, with downstream identification isolated in a separate conditional corollary. The spectral result now concerns the rank of a residual footprint under an explicit low-rank factorization, and the acyclicity result is explicitly conditional on a supplied or correctly learned order. The revision also defines known-target, unknown-target, and domain-indexed regimes, adds equivalence-class and falsification requirements, renames SKFM as Spectral Kernel Flow Matching, and recalibrates the experimental discussion accordingly.

\clearpage
\appendix
\section{Random-DAG Visual Diagnostics}
\label{app:random-dag-visuals}

Figure~\ref{fig:bridge_hybrid_random10_examples} shows three representative ten-node random-DAG diagnostics from the \textsc{Bridge}-hybrid sweep in Table~\ref{tab:bridge_hybrid_random10}. Each row uses the same topological layout for the true DAG and the discovered diagnostic graph. Green arrows are true positives, red arrows are false positives, and dashed orange arrows are true arrows missed by the ordered local BIC step inside the \textsc{Bridge} mask. The examples include a high-performing seed, a typical high-recall seed, and the hardest retained seed in the reported sweep.

\begin{figure}[p]
    \centering
    \includegraphics[width=\textwidth]{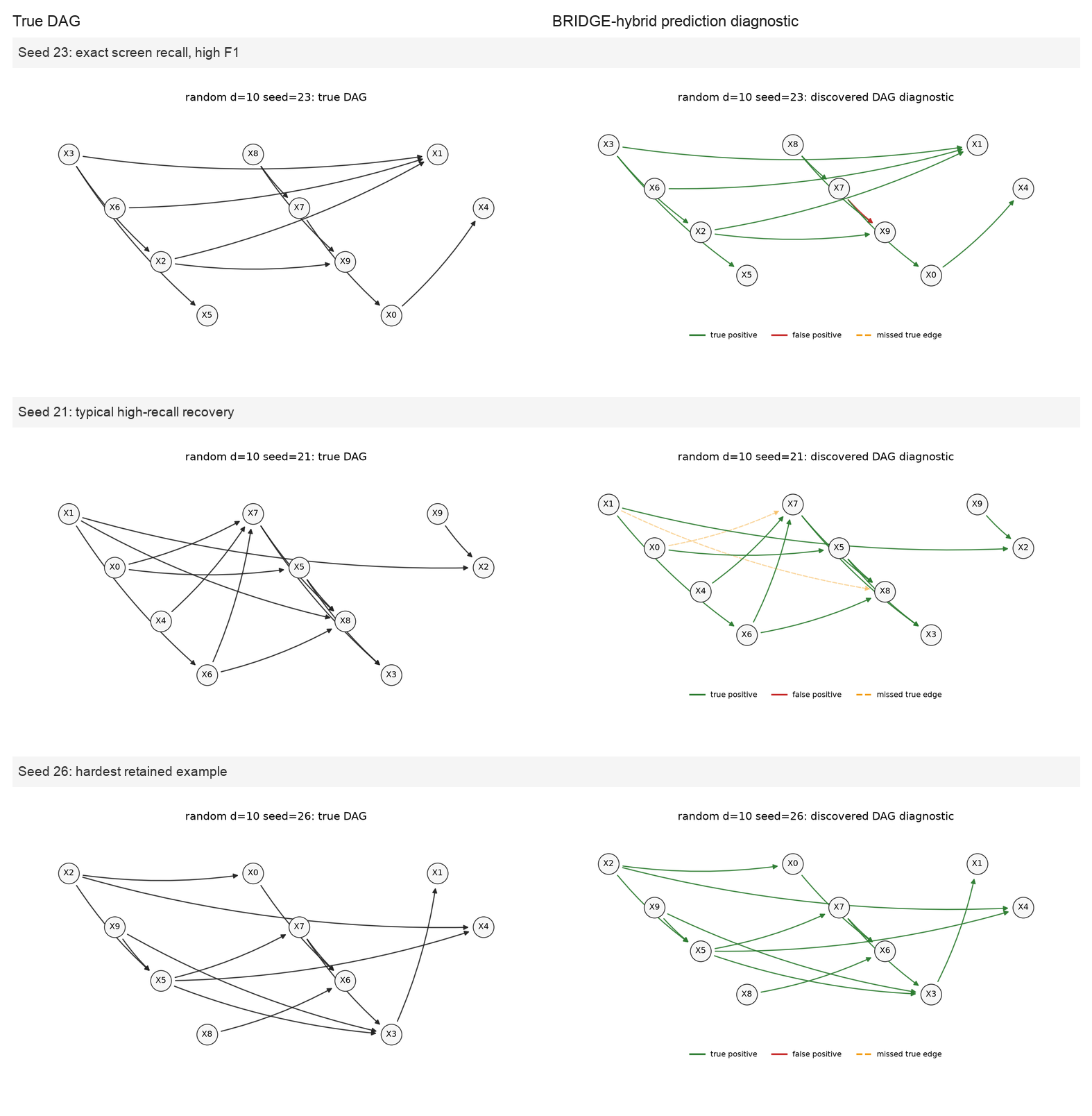}
    \caption{Representative \textsc{Bridge}-hybrid random-DAG visual diagnostics. Left panels show the sampled true ten-node DAGs; right panels show the corresponding discovered graphs, with true positives in green, false positives in red, and missed true edges as dashed orange arrows. Seed~23 illustrates a high-recall/high-$F_1$ case, seed~21 a typical recovery, and seed~26 the hardest retained example from the sweep.}
    \label{fig:bridge_hybrid_random10_examples}
\end{figure}

Figure~\ref{fig:bridge_hybrid_latent_random10_examples} gives the analogous diagnostic for the latent-confounded random-DAG extension. The directed accuracy metrics are still computed only against the visible direct-effect DAG; injected hidden common causes are displayed separately as dashed teal bidirected arcs between their observed children. This keeps the fully observed scalability claim distinct from the latent robustness/diagnostics question while making the data-generating confounding structure visible in the appendix.

\begin{figure}[p]
    \centering
    \includegraphics[width=\textwidth]{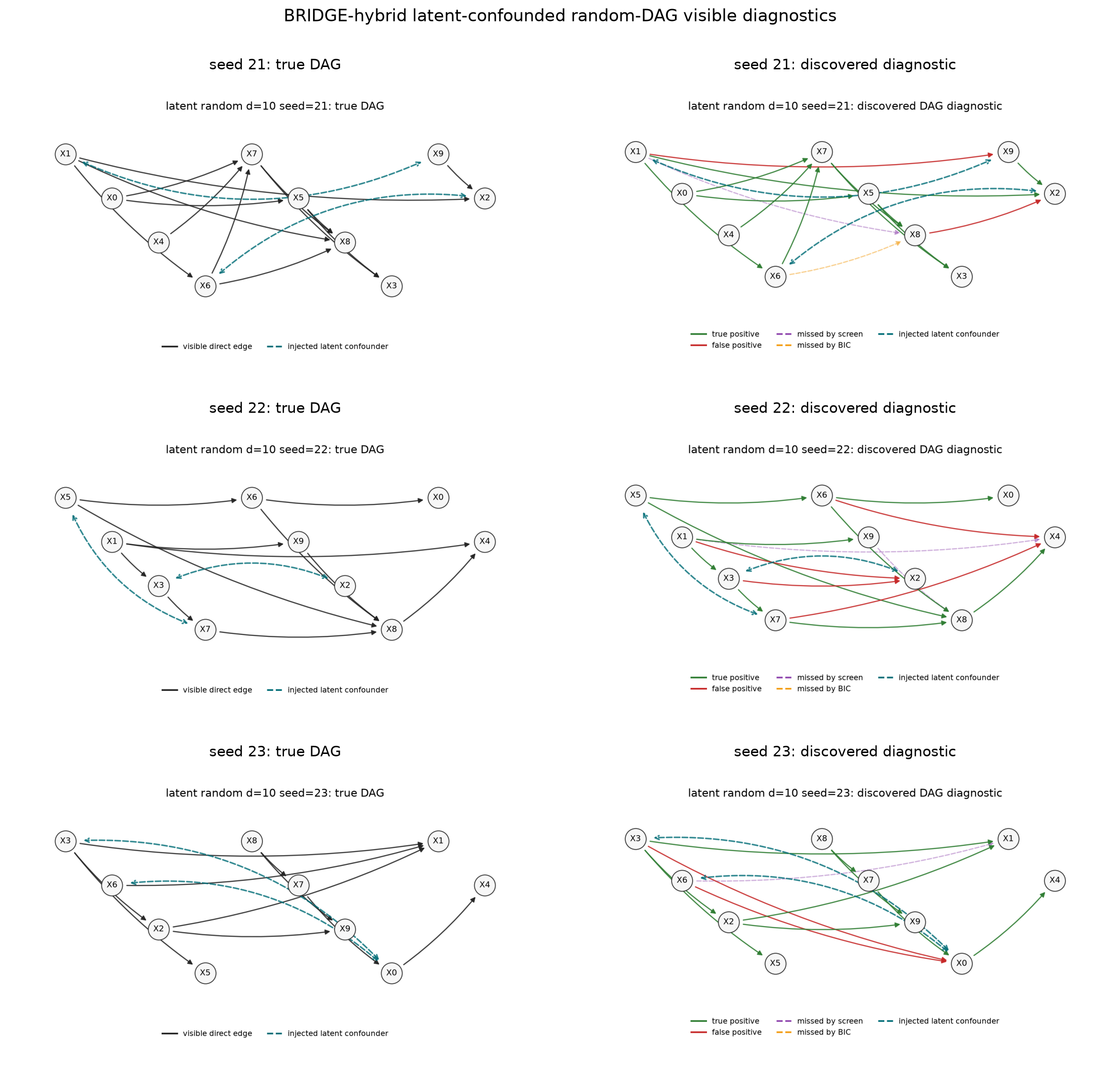}
    \caption{Latent-confounded \textsc{Bridge}-hybrid random-DAG visual diagnostics. Left panels show the visible direct-effect DAGs sampled by the generator, augmented with dashed teal bidirected arcs for injected hidden common causes. Right panels show the corresponding discovered visible graphs with the same true-positive, false-positive, and missed-edge coloring as Figure~\ref{fig:bridge_hybrid_random10_examples}. The latent arcs are diagnostic overlays rather than directed edges scored in the visible-DAG metrics.}
    \label{fig:bridge_hybrid_latent_random10_examples}
\end{figure}

\bibliography{allcitations,references}
\bibliographystyle{abbrvnat}

\end{document}